\pdfoutput=1
\documentclass{article}
\usepackage[utf8]{inputenc}
\usepackage[english]{babel}

\usepackage{natbib}
\setcitestyle{authoryear,open={(},close={)}} 

\usepackage[margin=3cm]{geometry}

\usepackage{amsmath,amsfonts,bm}

\def\eqref#1{equation~\ref{#1}}

\def\1{\bm{1}}

\DeclareMathAlphabet{\mathsfit}{\encodingdefault}{\sfdefault}{m}{sl}
\SetMathAlphabet{\mathsfit}{bold}{\encodingdefault}{\sfdefault}{bx}{n}

\usepackage{fancyhdr}
\usepackage{hyperref}
\usepackage{url}
\usepackage{url}            
\usepackage{booktabs}       
\usepackage{amsfonts}       
\usepackage{nicefrac}       
\usepackage{microtype}      
\usepackage{xcolor}         

\usepackage{pgfplots}
\usepackage{wrapfig}
\usepackage{multirow} 
\usepackage{enumitem}
\usepackage{algorithm}
\usepackage{algpseudocode}
\usepackage[para,online,flushleft]{threeparttable}

\usepackage{natbib}
\usepackage{booktabs}

\usepackage{mathtools}

\usepackage{arydshln}

\usepackage{microtype}
\usepackage[flushleft]{threeparttable}
\usepackage{subcaption}
\usepackage{booktabs} 
\usepackage{adjustbox}
\usepackage{minitoc}
\usepackage{titletoc}
\usepackage{amssymb}
\usepackage{bbm}
\usepackage{dsfont}
\usepackage{fancyhdr}

\usepackage{amsthm}

\newtheorem{theorem}{Theorem}[section]
\newtheorem{proposition}[theorem]{Proposition}
\newtheorem{lemma}[theorem]{Lemma}

\newtheorem{definition}[theorem]{Definition}

\newtheorem{remark}[theorem]{Remark}
\usepackage{pifont}

\pgfplotsset{compat=1.18}

\usepgfplotslibrary{groupplots}

\DeclareMathOperator{\CE}{CE}          

\newcommand{\M}{\texttt{M}}            

\newcommand{\calQ}{\mathcal{Q}}
\newcommand{\calW}{\mathcal{W}}

\title{\vspace{-3cm}\rule{\linewidth}{4pt}\\[6pt]
       Discrete Diffusion for Large Graph Generation via Structural Candidate Restriction
       \rule{\linewidth}{1pt}}
\author{
	Yassin Mohamadi \\
	  \small Email: \url{y.mohamadi@uva.nl}
	  \and
	Zeno Geradts\\
	\small Email: \url{z.geradts@nfi.nl}\\
    \and
    Marcel Worring \\
    \small Email: \url{m.worring@uva.nl}
}

\date{}

\begin{document}

\maketitle

\begin{abstract}
Synthesizing realistic graphs at scale is vital when the graphs of
interest are large and real-world samples are limited or access-sensitive. Diffusion-based generators have
recently driven much of the progress, offering high modeling capacity, but
most such methods have quadratic computational complexity and are
hence restricted to small-scale networks, currently up to 3k nodes.
Existing non-quadratic methods remain limited by memorization issues and a
trade-off between scalability and generation quality. Our goal is to
generate large graphs whose structural statistics --- e.g., degree distribution,
clustering, and path length --- faithfully reflect those of
real-world sparse graphs, without resorting to memorizing the training
data. We introduce a discrete graph diffusion model that restricts
training to a structurally motivated subset of node pairs --- observed
edges and their wedge non-edges --- reducing training complexity below quadratic in the number of nodes. To keep the noisy graph informative throughout both the forward and reverse trajectories, we design a three-class absorbing forward process governed by a \emph{degree-aware, floored} cosine noise schedule: unlike a structure-blind schedule that adds noise to every pair identically, ours adapts to each node's degree and never fully erases the graph's structure, keeping the noisy graph informative at every step. Experiments across diverse datasets show that our model consistently ranks among the top methods for structural fidelity against existing discrete diffusion baselines in large graph generation. Our code is available at this \href{https://github.com/yassinmihemedi/SCR-Discrete-Diffusion-for-Large-Graph-Generation-via-Structural-Candidate-Restrictionl.git}{repository}.
\end{abstract}

\section{Introduction}

Graphs in a natural way represent data in domains where a large number of entities interact in structured ways, such as financial networks, social platforms, and communication systems.  In many use cases, such as fraud detection using anomaly detection, infrastructure planning by structural analysis, and behavior modeling through scenario simulation, learning representations for the structural relationships is essential. When the structural data in real-world tasks are incomplete, sensitive, or costly to obtain, 
generating instances of representative large graphs is crucial.

Graph diffusion-based generative models (GDGMs) offer strong training
stability and architectural flexibility, but most target small graphs --
molecules and similarly-sized structures -- where the reverse denoising
process is computationally cheap
\citep{DBLP:conf/www/MoLP24,DBLP:conf/iclr/YouZPXTWS24,DBLP:conf/aaai/WenTOSYZCW24,DBLP:conf/nips/AyadiHSTG24}.
DiGress~\citep{Digress} and its generalized variant~\citep{wang2024large}
produce realistic topologies and node features, but only for graphs with
fewer than 300 nodes. Scaling GDGMs to large sparse graphs remains an open
problem: existing attempts limit each denoising step to a smaller
selection of nodes or edges \citep{EDGE,SparsDiff}, at the cost of
complexity that can still scale quadratically when that selection itself
grows with graph size; decouple structure and feature generation
\citep{Graphmaker}, assuming rich node features are not always available
in practice; or replace iterative denoising with autoregressive generation
over a compact representation \citep{k2tree,Higen,ARROW-Diff}. This last
approach shifts rather than removes the underlying cost: k2tree and Higen
still incur quadratic encoding/decoding and cross-community prediction
costs, respectively, while ARROW-Diff instead lacks formal convergence
guarantees.

We address two limitations of existing graph diffusion models on large,
sparse graphs. First, existing discrete diffusion schedules converge to a
stationary distribution that is structurally uninformative for sparse
data -- typically empty~\citep{EDGE} or fully dense~\citep{Digress,SparsDiff}
-- forcing the reverse process to reconstruct structure from one of these
two extremes. Second, since the overwhelming majority of node pairs in a
sparse graph are never connected, training over the full pair space is
dominated by low-signal absent-edge predictions. We address the first by
designing a noise schedule whose stationary distribution is instead a
sparse random graph matching the target degree sequence. The second we address by
restricting training to a compact, structurally motivated candidate set --
observed edges and their wedge-connected non-edges, together with a small
randomly sampled coverage set -- combined with a class-balance reweighting
scheme. Our contributions are as follows:
\begin{itemize}
    \item \textbf{Degree-aware noise schedule.} A masking schedule that
    reveals node pairs according to node degree keeps the noisy graph
    structurally informative throughout the reverse process.

    \item \textbf{Training with sub-quadratic complexity.} We restrict training
    to edges, their wedge non-edges, and a small random coverage
    set, and prove that, under chosen conditions, this yields an unbiased
    estimator of the full-pair objective with substantially improved
    signal-to-noise, reducing training cost from $O(N^2)$ to
    $O(N\mu_2)$, where $\mu_2$ is the second moment of the degree
    distribution.

    \item \textbf{Validated structural fidelity.} We evaluate our model on generative tasks
    across datasets of different sizes and domains, achieving consistently
    strong large-graph structural fidelity relative to other diffusion
    models, and though small-graph
    generation is not our \textit{primary} objective still reasonable performance on small graphs.
    
\end{itemize}

\section{Related work}
Graph generation has been a long-studied problem in statistics and network
science, well before the emergence of deep generative models. Classical
statistical models -- the Erd\H{o}s--R\'enyi model~\citep{ER}, the
Stochastic Block Model~\citep{SBM}, Chung-Lu~\citep{chung2006complex}, and
the more general Exponential Random Graph Model
(ERGM)~\citep{frank1986markov} -- generate graphs from a fixed probabilistic
form fit to observed statistics, ranging from simple edge-independent
Bernoulli models (ER, SBM, Chung-Lu) to ERGM, which allows edge dependence
via sufficient statistics such as triangle counts at the cost of an
intractable normalizing constant requiring MCMC for both fitting and
sampling. These models motivate our approach; for example, we adopt Chung--Lu's degree-based edge probabilities as a terminal distribution.

Neural-based models learn a representation of the graph
distribution via gradient-based training. Early approaches such as
GVAE~\citep{GVAE} and NetGAN~\citep{Netgan} adapted variational and
adversarial generative frameworks to graphs. Autoregressive models, such as GraphRNN~\citep{Graphrnn}, generate graphs
sequentially by modeling the adjacency matrix under a fixed node ordering,
but typically incur at least $O(N^2)$ complexity, limiting scalability. BiGG~\citep{BiGG} and its
weighted variant BiGG-E~\citep{BiGG-E} instead achieve $O((N+m)\log N)$ via
a recursive binary-tree decomposition of the adjacency matrix.
\citet{k2tree} compress the adjacency matrix into a $K^2$-tree sequence and
generate it autoregressively with a transformer, achieving a sub-quadratic
sequence length; however, encoding a graph into this representation and
decoding a generated sequence back into an adjacency matrix are each still
$O(N^2)$ operations, shifting the quadratic cost to a pre/post-processing
step rather than eliminating it.

Recently, graph diffusion-based generative models (GDGMs) have emerged as a
promising direction, with \emph{continuous}
variants~\citep{jo2022score, niu2020permutation}, which map graphs into
Euclidean space for noise injection and denoising, and \emph{discrete}
variants~\citep{xu2024discrete, haefeli2022diffusion}, which inject noise
directly into the graph's discrete structure. Applying discrete GDGMs to
large sparse graphs surfaces two distinct problems, not always shared by
every method. First, whether intermediate noisy graphs grow denser as the
forward process progresses depends on the transition's marginal: under a
uniform, resampling-style transition with a dense marginal -- as in
DiGress -- absent pairs are resampled into edges far more often than
existing edges are removed, so sparse graphs trend toward becoming dense as
$t\to T$, forcing the denoiser to consider a quadratic number of potential
edges at intermediate steps. This does not hold for absorbing,
deletion-only processes, whose forward trajectory only removes edges and
thus never increases density. Second, and independently of the first,
since most node pairs are absent edges even in the clean graph, training,  regardless of process
type, is dominated by low-signal absent-edge predictions degrading the signal-to-noise ratio during optimization.

To address the quadratic complexity of denoising over all $\binom{N}{2}$
node pairs, several methods restrict what each denoising
step computes. EDGE~\citep{EDGE} restricts edge
prediction to $K$ active nodes per step, but its transformer-based attention makes this cost
quadratic in $K$, $O(\max(m,K^2))$, and its empty-graph stationary distribution still
forces the reverse process to reconstruct structure from scratch.
SparseDiff~\citep{SparsDiff} samples a fixed-ratio subset of query edges
per step, giving $O(m)$ training cost, but total compute per diffusion
step at sampling time still covers all $\binom{N}{2}$ pairs, and, like
DiGress, its Erd\H{o}s--R\'enyi stationary distribution is far denser
than most real graphs. ARROW-Diff~\citep{ARROW-Diff} avoids the full
adjacency matrix altogether by diffusing over sampled random walks with
a separate GNN-based edge classifier, but lacks formal convergence
guarantees and can become unstable over long sequences.
GraphMaker~\citep{Graphmaker} instead improves efficiency via decoupled,
asynchronous mini-batch training of structure and features, but assumes
rich node features not always available outside attributed-graph domains.
\section{Preliminaries and Background}
\textbf{Graph representation.} We represent a graph as $G=(V,E,\mathbf X,\mathbf Y)$,
with node set $V$ ($|V|=N$) and edge set $E$. Node attributes are given by
$\mathbf X\in\{0,1\}^{N\times K_X}$, the stacked one-hot node
representation over an alphabet of size $K_X$, and edge attributes by
$\mathbf Y\in\{0,1\}^{N^2\times K_Y}$, the stacked one-hot edge
representation over alphabet $\mathcal Y$ of size $K_Y$, with single-step
transition entries $[\mathbf Q_t^Y]_{ab}=q_e(e_t{=}b\mid e_{t-1}{=}a)$,
$a,b\in\mathcal Y$. When no edge attribute information is modeled --
i.e.\ edges carry only existence, not type -- $\mathbf Y$ reduces to the
adjacency matrix $A\in\{0,1\}^{N\times N}$ ($K_Y{=}2$: no-edge/edge), where
$A_{ij}=1$ iff an edge exists between $i$ and $j$. This is the setting our
own model operates in.

\textbf{Graph diffusion process.} Graph diffusion generative models (GDGMs)
define a forward process $q(G_t\mid G_{t-1})$, $t=1,\dots,T$, that
gradually corrupts $G_0$ into a noisy graph $G_T$, and train a denoiser
$p_\theta(G_{t-1}\mid G_t)$  iteratively generating graphs from $G_T$ to $G_0$ to reverse it. In discrete graph diffusion, this is
typically written as \(q(G_t \mid G_0) \propto (\mathbf{X}\bar{\mathbf{Q}}_t^{X},\, \mathbf{Y}\bar{\mathbf{Q}}_t^{Y})\), 
with cumulative transition $\bar{\mathbf Q}_t^Y=\mathbf Q_1^Y\cdots\mathbf
Q_t^Y$ (analogously for $\mathbf Q_t^X,\bar{\mathbf Q}_t^X$). A common
parameterization is $\mathbf Q_t=\alpha_t\mathbf I +
(1-\alpha_t)\mathbf 1\mathbf p^\top$, preserving the current state with
probability $\alpha_t$ and resampling from marginal $\mathbf p$ otherwise.
The reverse process predicts from \(t\) to \(t-1\) via
\begin{equation}
p_{_\theta}^{x}(x_{t-1}^i\mid G_t) = \sum_{x^i\in\mathcal X} q^x(x_{t-1}^i\mid x_t^i,x^i,G_t)\,p_{_\theta}^x(x^i\mid G_t),
\end{equation}
and analogously for edges. 
\begin{figure*}[t!]
    \centering
 \begin{minipage}[t]{\textwidth}
     \includegraphics[width=\textwidth]{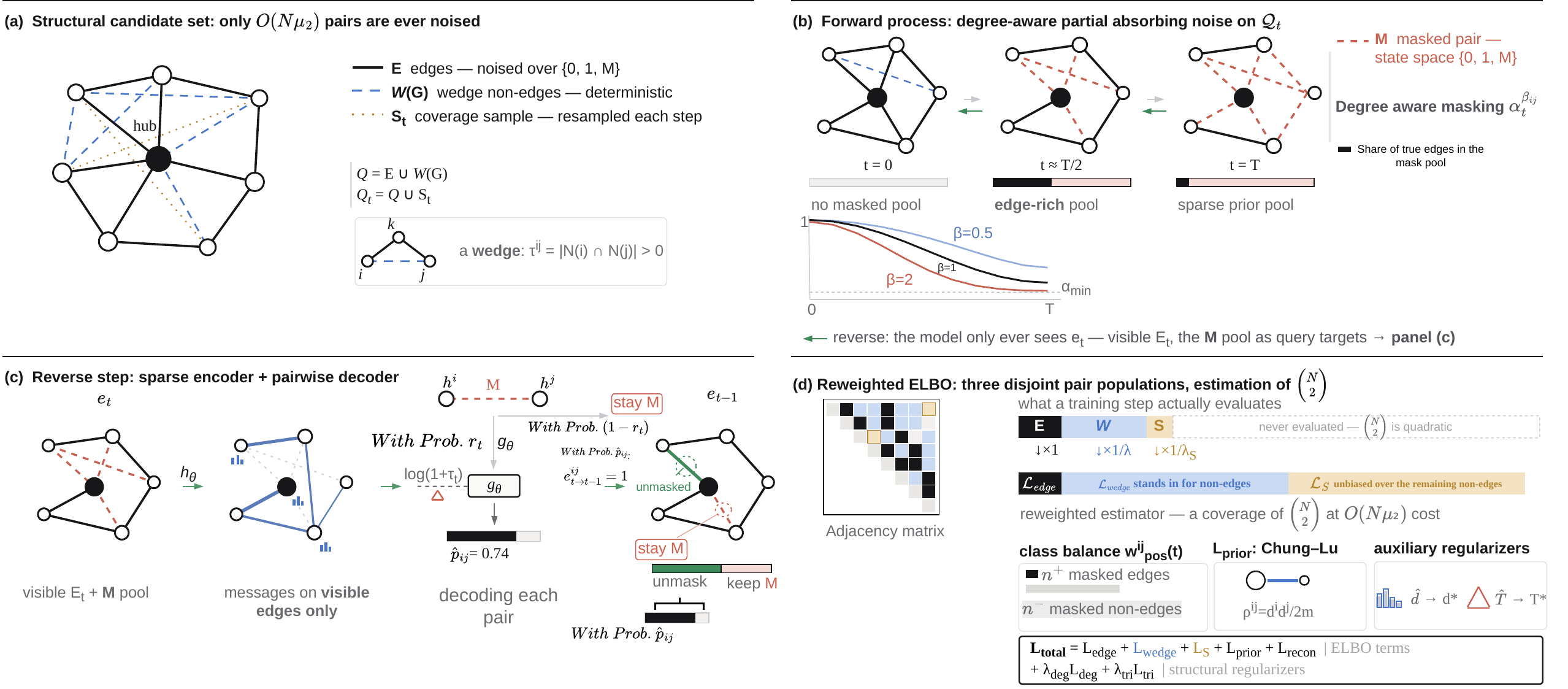}
    \caption{\small  Overview of the proposed method. (a) Diffusion is confined to the structural candidate set. (b) Pairs in \(Q_t\) follow a partial absorbing process over \{0, 1, M\} with a degree-aware masking exponent \(\beta^{ij}\). (c) The predictor passes messages only over visible edges \(E_t\), then decodes each queried pair independently. (d) The three pair populations are recombined, giving an estimator that covers all \(\binom N2 \) pairs at \(O(N\mu_2)\) cost per step.} \label{fig:model}
     \end{minipage}
\end{figure*}

\section{Method}\label{Sec:method}
\textbf{Problem statement.} Given a dataset of large, sparse graphs
$\{G_0^{(k)}\}$ (a single large, sparse graph $G_0$ being the special case
of a dataset with one element), our goal is a discrete diffusion model
that generates new graphs whose structural statistics --- e.g., degree
distribution, clustering, path length, among others --- faithfully match
those of the training data without memorizing any $G_0^{(k)}$, and that
scales to large graphs by achieving sub-quadratic training complexity in
$N$, in contrast to the $O(N^2)$ cost of processing all node pairs.

We propose that GDDMs can be scaled to large, sparse graphs through two complementary design choices: a forward process that restricts training to a compact, structurally meaningful subset of node pairs rather than all $\binom{N}{2}$, and a denoiser conditioned on local structural signal rather than edge state alone. We condition the denoiser
on the common-neighbor count, $\tau^{ij}$, a local closure
statistic related to clustering coefficient that gives a direct, per-pair
signal of local topology at every diffusion step --
following, in spirit, the ERGM view that conditioning generation on
structural statistics, rather than treating edges as independent, better
captures the dependencies real graphs exhibit. Since $\tau^{ij}=0$ by
definition for any pair with no shared neighbor, such pairs carry no
additional information beyond this known zero value, so we restrict the
diffusion process to a compact candidate set $\calQ$ -- observed edges and
pairs with $\tau^{ij}>0$, improving
signal-to-noise during training and reducing computational cost. We
define:

\begin{definition}[Wedge-Completion Set]
\label{def:W}
The wedge-completion set is the set of non-edges with at least one common
neighbour,
\begin{equation}
    \calW(G) =
    \big\{
        (i,j)\notin E
        \;:\;
        \exists\, k \text{ s.t. } (i,k)\in E,\, (j,k)\in E
    \bigr\}.
    \label{eq:W}
\end{equation}
\end{definition}
The structural candidate set is:
\begin{equation}
    \calQ(G) \;=\; E \;\cup\; \calW(G).
    \label{eq:Q}
\end{equation}

During training, the candidate set is extended at each step to
\begin{equation}\label{eq:Qt}
\calQ_t \;:=\; \calQ \cup S_t,
\end{equation}
where $S_t$ is a set of non-edges resampled uniformly at
each step $t$ (Eq.~\ref{eq:lambda}). Fig.~\ref{fig:model} provides an overview of our
model. 

In the following, we introduce the components of our diffusion model, namely, the forward and reverse process and the objective function. All algorithms, from optimization to sampling, are available in Appendix~\ref{appx:training_alg}.

\paragraph{Forward Process: Partial Absorbing Noise on $\calQ_t$.}

Pairs \((u,v)\in \bar{E}\) and  outside $\calQ_t$ are fixed to $0$ (non-edge) throughout.
For pairs $(i,j)\in\calQ_t$ the edge state space is
$\mathcal{S}=\{0,1,M\}$, where $M$ is a mask token.
\paragraph{Noise schedule and degree-aware masking.}
We want the noisy graph to stay sparse and non-empty at every diffusion
step, so the reverse process always begins from a partially informative
state rather than an empty or fully masked one. Two mechanisms achieve
this: restricting training to the candidate set $\calQ$ keeps the graph
sparse, and a floor $\alpha_{\min}>0$ on the noise schedule keeps it
non-empty, giving the \emph{floored cosine} schedule:
\begin{equation}
    \bar{\alpha}_t
    \;=\;
    \alpha_{\min}
    +
    (1-\alpha_{\min})
    \cos^2\!\!\left(\frac{\pi}{2}\cdot\frac{t}{T}\right),
    \quad
    \alpha_0 = 1,\quad
    \alpha_T = \alpha_{\min} \in (0,1).
    \label{eq:schedule}
\end{equation}
Within this sparse, non-empty pool, we additionally mask high-degree pairs
faster than low-degree ones, giving the denoiser more reverse-process
steps to recover low-degree connectivity, which a uniform masking rate
would otherwise resolve too quickly. For any pair $(i,j)$, we define:
\begin{equation}
    \beta^{ij}
    \;=\;
    \min\!\left(
        \beta_{\max},\;
        \left(\frac{d^{i*}\, d^{j*}}{\bar{d}^{\,2}}\right)^{\!\gamma/2}
    \right),
    \quad
    \gamma,\beta_{\max}>0.
    \label{eq:beta}
\end{equation}
Since $\beta^{ij}$ depends only on node degrees, it is cached for
$(i,j)\in\calQ$ (Alg.~\ref{alg:ergm_precompute}) and computed on demand
for $(i,j)\in S_t$ (Alg.~\ref{alg:ergm_ddpm_short}).

\paragraph{Forward marginal.}
The forward marginal distribution is:
\begin{align}
    q(e^{(i,j)}_t = e^{(i,j)}_0 \mid e^{(i,j)}_0)
    &\;=\; \alpha_t^{\,\beta^{ij}},
    \label{eq:fwd-retain}
    \\
    q(e^{(i,j)}_t = \M \mid e^{(i,j)}_0)
    &\;=\; 1 - \alpha_t^{\,\beta^{ij}}.
    \label{eq:fwd-mask}
\end{align}
The Markov structure gives the one-step transition:
\begin{equation}
    q(e^{(i,j)}_t\mid e^{(i,j)}_{t-1})
    \;=\;
    \frac{\alpha_t^{\,\beta^{ij}}}{\alpha_{t-1}^{\,\beta^{ij}}}\,\delta(e^{(i,j)}_t, e^{(i,j)}_{t-1})
    +
    \left(1 - \frac{\alpha_t^{\,\beta^{ij}}}{\alpha_{t-1}^{\,\beta^{ij}}}\right)
    \delta(e^{(i,j)}_t, \M).
    \label{eq:one-step}
\end{equation}
\paragraph{Reverse Process and Parametrisation.}
Following D3PM \citep{D3PM}, we parametrise the reverse process via a
neural network $f_\theta$ that outputs, independently for each pair
$(i,j)\in\calQ_t$, a Bernoulli distribution over the predicted class
$\tilde e_0^{ij}\in\{1, 0\}$; the mask state is
never a target of this prediction, since it is instead governed by the
schedule-only unmask probability $r_{t\to t_{\rm prev}}^{ij}$ below. Thus
$f_\theta(\tilde{\mathbf e}_0;\mathbf e_t,t)=\prod_{(i,j)\in\calQ_t}
f_\theta(\tilde e_0^{ij};\mathbf e_t,t)$ is a valid distribution over the
full edge-state vector, and
\begin{equation}
    p_\theta(\mathbf{e}_{t-1}\mid\mathbf{e}_t)
    \;=\;
    \sum_{\tilde{\mathbf{e}}_0}
    q(\mathbf{e}_{t-1}\mid\mathbf{e}_t,\tilde{\mathbf{e}}_0)\;
    f_\theta(\tilde{\mathbf{e}}_0;\mathbf{e}_t,t).
    \label{eq:reverse}
\end{equation}
Concretely, $f_\theta$ decomposes into a sparse graph encoder $h_\theta$
and a pairwise decoder $g_\theta$:
\begin{align}
\mathbf H_t &\;=\; h_\theta\bigl(\mathbf{X}_t, E_t, \{p^{ij}_{\rm vis}\}_{(i,j)\in E_t},\, \mathbf d^*,\, \mathbf d_t\bigr)
\label{eq:encoder}\\[4pt]
\hat p_{ij} &\;=\; g_\theta\bigl(\mathbf{h}^i_t, \mathbf{h}^j_t,\, \log(1+\tau^{ij}_t)\bigr), \qquad (i,j)\in\calQ_t
\label{eq:decoder}
\end{align}
where $E_t=\{(i,j):e^{(i,j)}_t=1\}$ is the sparse visible edge set
(message passing is restricted to $E_t$); $p^{ij}_{\rm vis}\in(0,1)$
weights each message during convolution; $\mathbf d^*=\{d^{i*}\}$ is the
true target degree sequence, supplied as a permanent per-node
conditioning feature (distinct from its role as a regression target in
$\mathcal L_{\rm deg}$ below); $\mathbf d_t$ is the current visible degree
under $E_t$. During training, pairs outside $\calQ_t$ are not queried and remain fixed
non-edges throughout. 

For the absorbing process the posterior is analytic:
\begin{equation}\small
    q(e^{(i,j)}_{t_{\rm prev}}\mid e^{(i,j)}_t,e^{(i,j)}_0)
    \;=\;
    \begin{cases}
        r_{t\to t_{\rm prev}}^{ij}\,\delta(e^{(i,j)}_0)
        +
        (1-r_{t\to t_{\rm prev}}^{ij})\,\delta(\M)
        & \text{if } e^{(i,j)}_t=\M,
        \\
        \delta(e^{(i,j)}_t)
        & \text{if } e^{(i,j)}_t\neq\M,
    \end{cases}
    \label{eq:posterior}
\end{equation}
where the unmask probability for a transition from $t$ to $t_{\rm prev}$ is
\begin{equation}
    r_{t\to t_{\rm prev}}^{ij}
    \;=\;
    \frac{\alpha_{t_{\rm prev}}^{\,\beta^{ij}} - \alpha_t^{\,\beta^{ij}}}
         {1 - \alpha_t^{\,\beta^{ij}}},
    \qquad
    0 \;\leq\; r_{t\to t_{\rm prev}}^{ij} \;\leq\; 1.
    \label{eq:unmask}
\end{equation}
When $t_{\rm prev}=t-1$, Eq.~\ref{eq:unmask} gives the single-step
probability; at the final step ($t_{\rm prev}=0$, $\alpha_0=1$),
$r_{t\to 0}^{ij}=1$, so all remaining masked pairs are decoded.
Alg.~\ref{alg:ergm_sample_step} presents the sampling procedure.
\paragraph{Deriving the loss function.}
To optimize the parameters of our model, we use the evidence lower bound
(ELBO) as our training loss. Since the candidate edge set is restricted to
$\calQ$, only a subset of non-edges is explicitly considered during
training. To obtain an \emph{unbiased} estimate of the ELBO over all
$\binom{N}{2}$ node pairs, we additionally draw a random coverage sample
$S_t$ of non-edges at each step (Eq.~\ref{eq:Qt}) and apply a
Horvitz--Thompson (HT) correction, $1/\lambda_S$, since $S_t$ is drawn
uniformly at random with a known inclusion probability $\lambda_S$. The wedge non-edges in $\calQ$, by contrast, are included \emph{deterministically}
rather than sampled, so no probabilistic correction applies to them; instead,
their aggregate loss is rescaled by $1/\lambda$ (Eq.~\ref{eq:lambda}). Summing over $E$, the $\lambda$-rescaled wedge non-edges, and the
HT-corrected coverage sample $S_t$ yields a tractable, unbiased estimator
of the full ELBO; Alg.~\ref{alg:ergm_ddpm_short} summarizes the
optimization procedure. To further correct the severe imbalance between
edges and non-edges, each cross-entropy term is additionally
scaled by a time-dependent edge-presence weight $w_{\mathrm{pos}}^{ij}(t)$
(Eq.~\ref{eq:wpos}). Combining both corrections gives the final training
objective:
\begin{equation}
\mathcal{L}(\theta)
=
\mathcal{L}_{\rm edge}
+ \mathcal{L}_{\rm wedge}
+ \mathcal{L}_{S}
+ \mathcal{L}_{\rm prior}
+ \mathcal{L}_{\rm recon}
+ \lambda_{\rm deg}\,\mathcal{L}_{\rm deg}
+ \lambda_{\rm tri}\,\mathcal{L}_{\rm tri}
\label{eq:elbo_full}
\end{equation}
where $\mathcal{L}_{\rm edge}$, $\mathcal{L}_{\rm wedge}$, and $\mathcal{L}_{S}$ are masked cross-entropy losses evaluated over true edges, wedge candidates, and a random coverage sample, respectively:
\begin{align}
\mathcal{L}_{\rm edge}
&=
\frac{1}{T}\sum_{t=1}^{T-1}
\mathbb{E}_{\mathbf{e}_t|\mathbf{e}_0}
\!\left[
    \sum_{(i,j)\in E}
    w^{ij}_{\rm pos}(t)\,
    r_t^{ij}\,\mathbf{1}[e^{(i,j)}_t{=}\mathbf{M}]\,
    \mathrm{CE}\!\bigl(e^{(i,j)}_0,\hat p_{ij}\bigr)
\right]
\label{eq:L_edge}\\[4pt]
\mathcal{L}_{\rm wedge}
&=
\frac{1}{T}\sum_{t=1}^{T-1}
\mathbb{E}_{\mathbf{e}_t|\mathbf{e}_0}
\!\left[
    \frac{1}{\lambda}
    \sum_{(i,j)\in\calW}
    w^{ij}_{\rm pos}(t)\,
    r_t^{ij}\,\mathbf{1}[e^{(i,j)}_t{=}\mathbf{M}]\,
    \mathrm{CE}\!\bigl(e^{(i,j)}_0,\hat p_{ij}\bigr)
\right]
\label{eq:L_wedge}\\[4pt]
\mathcal{L}_{S}
&=
\frac{1}{T}\sum_{t=1}^{T-1}
\mathbb{E}_{S_t}\,\mathbb{E}_{\mathbf{e}_t|\mathbf{e}_0}
\!\left[
    \frac{1}{\lambda_S}
    \sum_{(i,j)\in S_t}
    w^{ij}_{\rm pos}(t)\,
    r_t^{ij}\,\mathbf{1}[e^{(i,j)}_t{=}\mathbf{M}]\,
    \mathrm{CE}\!\bigl(e^{(i,j)}_0,\hat p_{ij}\bigr)
\right]
\label{eq:L_S}
\end{align}
rescaled by $1/\lambda$ and $1/\lambda_S$ respectively to correct for subsampling, so all three terms estimate the reconstruction loss over the full edge set at comparable scale.  $\mathcal{L}_{\rm deg}$ and $\mathcal{L}_{\rm tri}$ regularize node degree and triangle count toward their true values $d^{i*}$ and $T^*$ (defined below):
\begin{align}
\mathcal{L}_{\rm deg}
&=
\frac{1}{T}\sum_{t=1}^{T-1}
\mathbb{E}_{\mathbf{e}_t|\mathbf{e}_0}
\!\left[
    \frac{1}{N}\sum_{i=1}^{N}
    \Bigl(\hat{d}^{i}(\mathbf{e}_t) - d^{i*}\Bigr)^2
\right]
\label{eq:L_deg}\\[4pt]
\mathcal{L}_{\rm tri}
&=
\frac{1}{T}\sum_{t=1}^{T-1}
\mathbb{E}_{\mathbf{e}_t|\mathbf{e}_0}
\!\left[
    \left(
        \frac{\hat{T}_\theta(\mathbf{e}_t) - T^*}{\max(T^*,\,1)}
    \right)^{\!2}
\right]
\label{eq:L_tri}
\end{align}
with
\begin{equation}
d^{i*} = \sum_{(i,j)\in \calQ} e^{(i,j)}_0, \quad
\tau^{ij}_t = \bigl|\mathcal{N}(i;G_t)\cap\mathcal{N}(j;G_t)\bigr|, \quad
T^* = T(G_0) = \frac13\sum_{(i,j)\in E}\tau^{ij}_0,
\label{eq:quantities}
\end{equation}
the true degree of node $i$ over candidate pairs $\calQ$, the common-neighbor count of $(i,j)$ on $G_t$, and the triangle count of $G_0$. $d^{i*}$ is shared, unmodified, across all its uses in this section: as an encoder input feature (Eq.~\ref{eq:encoder}), as the regression target of $\mathcal{L}_{\rm deg}$, and as a component of the Chung--Lu prior $\rho^{ij}$ (Eq.~\ref{eq:LT}). The predicted counterparts are
\begin{equation}
\hat{d}^{i}(\mathbf{e}_t) = \sum_{(i,j)\in \calQ_t}\hat p_{ij}, \qquad
\hat{T}_\theta(\mathbf{e}_t) = \frac{1}{3}\sum_{(i,j)\in \calQ_t}\hat p_{ij}\;\tau^{ij}_t.
\label{eq:pred_quantities}
\end{equation}
$\lambda$ is the fraction of non-edges falling in the wedge set, and $\lambda_S$ the known random inclusion rate of the coverage sample $S_t$:
\begin{equation}
\lambda = \frac{|\calW(G)|}{\binom{N}{2}-m} \in (0,1], \qquad
\lambda_S = \frac{|S_t|}{\binom{N}{2}-m-|\calW(G)|}.
\label{eq:lambda}
\end{equation}

\paragraph{Reconstruction term.}
$\mathcal{L}_{recon}= \sum_{(i,j)\in\calQ} w^{ij}\, w^{ij}_{\rm pos}(0)\,\CE\!\bigl(e^{(i,j)}_0, \hat p_{ij}\bigr)$
with $w^{ij}=1$ for edges and $w^{ij}=1/\lambda$ for wedge non-edges.
\paragraph{Prior term via Chung-Lu.}This term measures the mismatch between the forward marginal at $t=T$ and the Chung–Lu prior used to initialize sampling. $L_T$
has no masked pool and is therefore not scaled by $w^{ij}_{\rm pos}$.
\begin{equation}
    \mathcal{L}_{prior}
    \;=\;
    \sum_{(i,j)\in\calQ}
    \bar{\alpha}_T^{\,\beta^{ij}}\,
    \CE\!\bigl(e^{(i,j)}_0,\,\rho^{ij}\bigr),
    \quad
    \rho^{ij} = \min\!\left(1,\ \frac{d^{i*}d^{j*}}{2m}\right) \in [0,1].
    \label{eq:LT}
\end{equation}

\paragraph{Class-balance reweighting.}
To prevent the vast majority of non-edges from dominating the masked-pool loss, we reweight each term above by the current positive-to-negative class ratio. Let $y^{ij} := e^{(i,j)}_0 \in \{0,1\}$ denote the ground-truth edge label, and let
\begin{equation}
\mathbf{1}^{ij}_{\rm pool}(t) \;:=\;
\begin{cases}
\mathbf{1}[e^{(i,j)}_t = \M] & t>0 \\
1 & t=0
\end{cases}
, \qquad (i,j)\in\calQ_t,
\label{eq:pool_indicator}
\end{equation}
denote pool membership under the current mask. Since $S_t\subset\calQ_t$ consists exclusively of non-edges ($y^{ij}=0$ for $(i,j)\in S_t$), the positive count reduces to a sum over the static candidate set $\calQ$, while the negative count does not (maskable $S_t$ pairs can contribute):
\begin{align}
n^+_t &= \sum_{(i,j)\in\calQ} y^{ij}\,\mathbf{1}^{ij}_{\rm pool}(t), &
n^-_t &= \sum_{(i,j)\in\calQ_t} (1-y^{ij})\,\mathbf{1}^{ij}_{\rm pool}(t),
\label{eq:npos_nneg}\\[4pt]
\rho_t &= \frac{n^-_t}{\max(n^+_t,\,1)}, &
w^{ij}_{\rm pos}(t) &= \rho_t\, y^{ij} + (1-y^{ij}).
\label{eq:wpos}
\end{align}
$w^{ij}_{\rm pos}(t)$ is recomputed at every sampled $t$ from the current batch and is not cached per graph.

\paragraph{Analysis.}
Training already restricts each graph to its candidate set $\calQ$ rather
than the full $\binom{N}{2}$ pairs; controlling the size of the randomly
sampled coverage set $S_t$ (Eq.~\ref{eq:lambdaS_choice}) alongside this
restriction yields a combined candidate set of size $|\calQ_t|=O(N\mu_2)$
(Prop.~\ref{prop:coverage}), where $\mu_2 = \frac{1}{N}\sum_{v} d_v^2$ is
the second moment of the degree sequence --- a fraction $O(\mu_2/N)$ of
all possible pairs, which vanishes as $N\to\infty$ whenever $\mu_2=o(N)$.
During sampling, in contrast, the reverse process has no access to the
ground-truth candidate set $\calQ$ --- defined from the clean graph
$G_0$ --- and must resolve all $\binom{N}{2}$ pairs at each step,
introducing an unavoidable quadratic dependency on the number of nodes at
sampling time (Table~\ref{tab:complexity_diffusion_only}); full detail is
in Appendix~\ref{appx:Analysis}.
\paragraph{Improving SNR.}
Restricting training to $\calQ_t$ improves the class-mass
signal-to-noise ratio (Def.~\ref{def:snr}) over training on all
$\binom{N}{2}$ pairs: at fixed masking fraction, the unweighted SNR
improves from $\Theta(m/N^2)$ to $\Theta(m/N\mu_2)$ under $\calQ_t$
(Rem.~\ref{rem:snr-dense}), a gain of $\Theta(N/\mu_2)$. The
positive-class weight $w^+_b(t)$ (Eq.~\ref{eq:wpos}) then exactly balances
the remaining class imbalance within $\calQ_t$, achieving a class-mass
ratio of exactly $1$ regardless of the pool (Prop.~\ref{prop:snr}), a
further gain of $\Theta(N\mu_2/m)$. Together, restriction and reweighting
improve the class-mass SNR by $\Theta(N^2/m)=\Theta(1/\rho)$ over an
unweighted, dense-pool baseline, where $\rho$ is the graph's edge density:
sparser graphs benefit more (Rem.~\ref{rem:snr-dense}). The
coverage-sampled term $S_t$ is HT corrected, so the full
objective remains an unbiased estimate of the full-graph objective
overall (Prop.~\ref{prop:unbiased}); full derivations are in
Appendix~\ref{appx:Analysis}.

\section{Experiments}
We evaluate our model's performance on graph generation to assess how realistic the generated graphs are and how well they match the distribution of the training data. To check the fidelity of the generated graphs, we consider characteristics related to the structure of the training data and compare them accordingly while also screening for memorization indicators to avoid mere duplication of the training data.

\begin{table*}[h]
\centering
\caption{Dataset statistics.}
\label{tab:data_status}
\small
\begin{adjustbox}{max width=\textwidth}
\begin{tabular}{lccccccccc}
\toprule
Statistic & Polblogs & Bitcoin-Alpha & Bitcoin-OTC & Cora & Citeseer & Pubmed & Community & Ego & ZINC \\
\midrule
Nodes  & 1,222  & 3,775  & 5,875  & 2,485 & 2,120 & 19,717 & [60, 160]   & [50, 399]   & [9, 37] \\
Edges  & 16,714 & 14,120 & 21,489 & 5,069 & 3,679 & 44,324 & [231, 1,956] & [56, 1,062] & [8, 41] \\
Graphs & 1      & 1      & 1      & 1     & 1     & 1      & 510         & 757         & 12,000  \\
\bottomrule
\end{tabular}
\end{adjustbox}
\end{table*}

\subsection{Datasets and metrics}
We consider a diverse range of graph datasets, from citation networks~\cite{sen2008collective} to financial networks~\cite{kumar2016edge,kumar2018rev2,EllipticPlusPlus}. Table~\ref{tab:data_status} reports dataset statistics; see Appendix~\ref{data_explanation} for details.

A generated graph should reproduce the statistical properties of the training distribution while remaining structurally distinct from any individual training graph --- i.e., matching the population, not memorizing any instance. We therefore evaluate distributional fidelity alongside explicit memorization checks. For fidelity, we compare generated and input graphs on statistics capturing both local and global structure. To capture local connectivity patterns and neighborhood cohesion, we use
Maximum Mean Discrepancy (MMD)~\cite{MMD} over node degree and clustering
coefficient distributions, respectively (lower is better). We also compare the power-law exponent (PLE) and characteristic path length (CPL)~\cite{chanpuriya2021power}, reflecting how well the model preserves scale-free degree behavior and small-world connectivity, along with the assortativity coefficient (AC), which captures the tendency of nodes to connect to others of similar degree. We further visualize the complementary cumulative distribution function (CCDF) of node degree, $P(\text{Degree}+1 \geq x)$ --- obtained by rank-sorting the pooled degree sequence and plotting on log-log axes --- to inspect how closely the heavy-tailed degree distribution of generated graphs tracks the training data (Fig.~\ref{fig:ap_node_degreeccd}). 

To rule out memorization, we report edge overlap (EO), since a high value can indicate the model is reproducing training graphs rather than generalizing~\cite{EDGE}. For smaller datasets, we additionally report uniqueness and novelty (Tab.~\ref{tab:small_metrics},~\ref{tab:zinc_metrics}) quantifying how many generated instances are unique and how many are novel relative to the training data. Implementation details are in Appendix\footnote{For baselines, we select efficient models capable of large-graph generation and exclude runs that result in out-of-memory failures.}~\ref{appx:implementation}.

\begin{figure*}[t]
\centering
\vspace{-0.3cm}
\begin{minipage}[t]{\textwidth}
\centering
\captionof{table}{\small
Statistics of generated vs. ground-truth graphs. Proximity to ground truth is optimal for all metrics except EO (where ground truth denotes self-overlap). Bold/underlined indicate best/second-best.
}\label{tab:main_result}
\vspace{-0.3cm}
\begin{adjustbox}{max width=\textwidth}
\begin{tabular}{l|ccccc|ccccc|ccccc}
\toprule
\multicolumn{1}{c}{} &
\multicolumn{5}{c|}{\textbf{Cora}} &
\multicolumn{5}{c|}{\textbf{Citeseer}} &
\multicolumn{5}{c}{\textbf{Pubmed}} \\
& EO$\downarrow$ & PLE & CPL & CC & AC
& EO$\downarrow$ & PLE & CPL & CC & AC
& EO$\downarrow$ & PLE & CPL & CC & AC \\
\midrule
ground truth
& 100.00 & 1.89 & 6.36 & 0.090 & -0.07
& 100.00 & 2.07 & 9.35 & 0.130 & 0.01
& 100.00 & 2.18 & 6.31 & 0.050 & -0.04 \\
\hdashline
GraphRNN~\citeyear{Graphrnn}
& \underline{0.01} & \underline{1.77} & 4.80 & \underline{0.040} & -0.01
& \underline{0.10} & \underline{2.30} & 4.10 & 0.002 & 0.23
& \underline{0.10} & 1.59 & 3.30 & \underline{0.200} & 0.10 \\
BiGG~\citeyear{BiGG}
& 10.40 & 1.50 & 3.10 & 0.005 & 0.10
& 12.40 & 1.60 & 3.40 & 0.004 & 0.11
& 17.30 & 1.55 & 3.40 & 0.350 & 0.09 \\
EDGE~\citeyear{EDGE}
& 2.10 & 1.76 & \underline{4.90} & 0.030 & \underline{-0.04}
& 6.10 & 1.72 & \underline{5.80} & 0.010 & \underline{0.05}
& 8.20 & \underline{1.60} & \underline{4.23} & 0.220 & \textbf{0.02} \\
ARROW-Diff~\citeyear{ARROW-Diff}
& 41.20 & 1.66 & 4.80 & 0.001 & \textbf{-0.08}
& 64.50 & 1.77 & 5.40 & \underline{0.062} & \textbf{0.01}
& 51.30 & 1.58 & 3.27 & 0.370 & 0.12 \\
SCR (ours)
& \textbf{0.00} & \textbf{1.83} & \textbf{5.10} & \textbf{0.076} & 0.10
& \textbf{0.00} & \textbf{2.04} & \textbf{6.20} & \textbf{0.067} & 0.35
& \textbf{0.00} & \textbf{1.77} & \textbf{4.50} & \textbf{0.180} & \underline{0.03} \\
\bottomrule
\end{tabular}
\end{adjustbox}

\vspace{2pt}

\begin{adjustbox}{max width=\textwidth}
\begin{tabular}{l|ccccc|ccccc|ccccc}
\toprule
\multicolumn{1}{c}{} &
\multicolumn{5}{c|}{\textbf{Alpha}} &
\multicolumn{5}{c|}{\textbf{OTC}} &
\multicolumn{5}{c}{\textbf{Polblogs}} \\
& EO$\downarrow$ & PLE & CPL & CC & AC
& EO$\downarrow$ & PLE & CPL & CC & AC
& EO$\downarrow$ & PLE & CPL & CC & AC \\
\midrule
ground truth
& 100.00 & 1.92 & 3.58 & 0.080 & -0.17
& 100.00 & 1.97 & 3.60 & 0.060 & -0.16
& 100.00 & 1.41 & 2.74 & 0.230 & -0.22 \\
\hdashline
GraphRNN~\citeyear{Graphrnn}
& \underline{0.10} & 1.10 & 1.04 & 0.005 & 0.12
& \textbf{0.00} & 0.68 & 1.73 & 0.170 & 0.12
& \underline{1.60} & 1.26 & \underline{2.50} & 0.010 & 0.10 \\
BiGG~\citeyear{BiGG}
& 9.70 & 0.80 & 1.20 & 0.001 & 0.15
& 10.20 & 0.70 & 1.20 & \underline{0.150} & 0.13
& 19.30 & 1.19 & 2.20 & 0.010 & 0.10 \\
EDGE~\citeyear{EDGE}
& 3.70 & 1.10 & \underline{2.10} & 0.040 & \textbf{-0.05}
& \underline{5.60} & \underline{0.87} & 1.90 & 0.180 & \underline{0.10}
& 3.10 & \underline{1.37} & \textbf{2.60} & \textbf{0.200} & \underline{-0.18} \\
ARROW-Diff~\citeyear{ARROW-Diff}
& 43.80 & \underline{1.20} & 2.05 & \underline{0.050} & \underline{0.05}
& 53.70 & 0.82 & \underline{2.02} & 0.230 & 0.11
& 61.30 & 1.32 & 2.10 & \underline{0.170} & \textbf{-0.19} \\
SCR (ours)
& \textbf{0.00} & \textbf{1.80} & \textbf{3.60} & \textbf{0.100} & 0.10
& \textbf{0.00} & \textbf{1.82} & \textbf{3.65} & \textbf{0.120} & \textbf{0.07}
& \textbf{0.00} & \textbf{1.39} & \textbf{2.60} & \textbf{0.200} & -0.11 \\
\bottomrule
\end{tabular}
\end{adjustbox}

\end{minipage}
\end{figure*}
\vspace{-0.2cm}
\begin{figure*}[h]
    \centering

    \begin{subfigure}[t]{0.32\textwidth}
        \centering
        \includegraphics[width=0.7\textwidth]{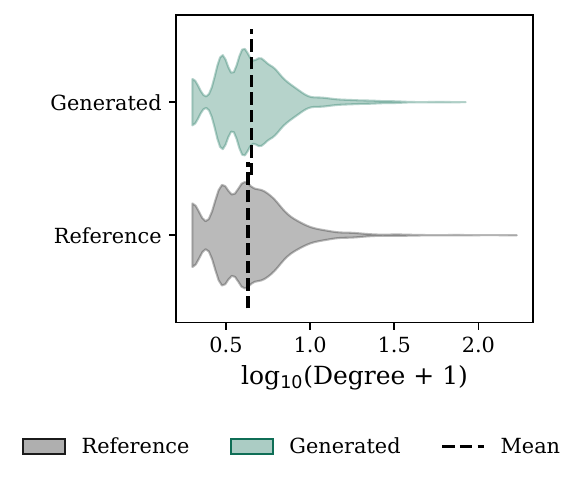}
        \caption{Cora}
        
    \end{subfigure}
    \hfill
    \begin{subfigure}[t]{0.32\textwidth}
        \centering
        \includegraphics[width=0.7\textwidth]{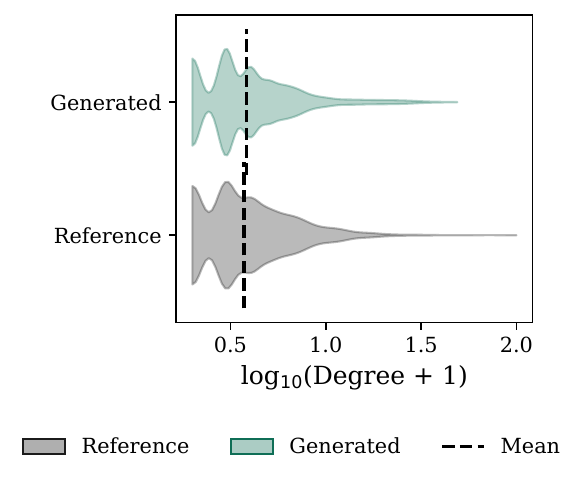}
        \caption{Citeseer}
           
    \end{subfigure}
    \hfill
    \begin{subfigure}[t]{0.32\textwidth}
        \centering
        \includegraphics[width=0.7\textwidth]{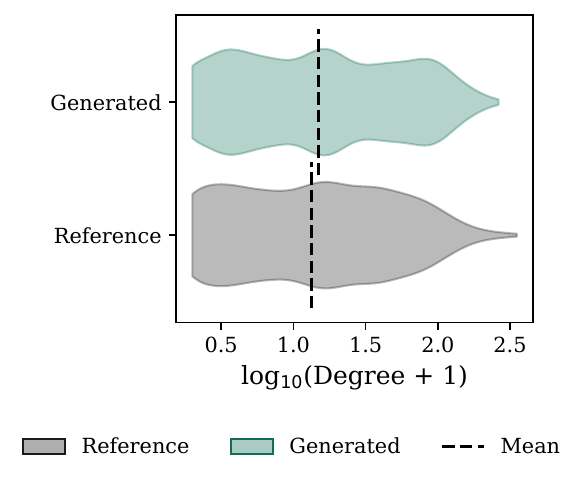}
        \caption{Polblogs}
        \end{subfigure}
    \caption{\small Degree distribution comparison between generated and training graphs, restricted to generated samples with zero edge overlap with the training data (i.e., no memorized edges). See Fig.~\ref{fig:ap_node_degree_vilon} for further datasets.}
    \label{fig:node_degree_v}
\end{figure*}

\paragraph{Results.}As reported in table~\ref{tab:main_result}, across the four degree- and closure-based statistics (EO, PLE, CPL, CC), our
model is consistently the closest or tied-closest to the ground truth on all six
datasets. The
clustering-coefficient result is worth highlighting specifically: our model
is closest or second-closest to ground truth on all six datasets, which is
consistent with the wedge-completion set $\calW(G)$ and the triangle-count
regularizer $\mathcal L_{\rm tri}$ both explicitly targeting local closure
structure — though we note this connection is suggestive rather than
established, since we do not isolate these components' individual
contribution to CC recovery here (see \S~\ref{sec:ablation}). Assortativity is the clear exception to this pattern --- performance is uneven across datasets, plausibly because every mechanism we introduce operates on local pairs or neighborhoods, with no component directly targeting this global degree correlation. Visualizing generated large graphs directly is uninformative at scale;
instead, we visualize degree distributions alongside the numerical results
(Fig.~\ref{fig:node_degree_v}). Further results, including those for smaller datasets (Tabs.~\ref{tab:small_metrics}, \ref{tab:zinc_metrics}) and empirical sampling time (Fig.~\ref{fig:time_sampling}), are provided in Appendix~\ref{sec:a_extend_result}.

\begin{table*}[t]
\centering
\small
\caption{\small Ablation of Prior Scaffolds and Regularizers. Proximity to Ground Truth is optimal. Bold/underlined indicate best/second-best within each group.}
\label{tab:ablation_consolidated}
\begin{adjustbox}{max width=0.9\textwidth}
\begin{tabular}{l|cccc|cccc|cccc}
\toprule
& \multicolumn{4}{c|}{\textbf{Cora}} & \multicolumn{4}{c|}{\textbf{Polblogs}} & \multicolumn{4}{c}{\textbf{Alpha}} \\
Variant & PLE & CPL & CC & AC & PLE & CPL & CC & AC & PLE & CPL & CC & AC \\
\midrule
Ground Truth & 1.89 & 6.36 & 0.090 & -0.07 & 1.41 & 2.74 & 0.230 & -0.22 & 1.92 & 3.58 & 0.080 & -0.17 \\
\midrule
\textit{Prior Scaffold} & & & & & & & & & & & & \\
\quad SCR (Ours) & \textbf{1.83} & \textbf{5.10} & \textbf{0.076} & \underline{0.10} & \underline{1.39} & \underline{2.60} & \underline{0.200} & \textbf{-0.11} & \textbf{1.80} & \textbf{3.60} & \textbf{0.100} & \underline{0.10} \\
\quad w/ Erd\H{o}s--R\'enyi & \textbf{1.83} & 4.87 & \underline{0.115} & 0.33 & 1.32 & 2.57 & 0.096 & 0.33 & 1.55 & \underline{3.65} & 0.640 & 0.43 \\
\quad w/ SBM & \underline{1.81} & \underline{4.92} & 0.005 & \textbf{0.01} & \textbf{1.40} & \textbf{2.66} & \textbf{0.204} & \underline{-0.08} & 1.58 & 3.09 & 0.558 & 0.27 \\
\cdashline{1-13}
\textit{Loss Terms} & & & & & & & & & & & & \\
\quad SCR (Ours) & \textbf{1.83} & \textbf{5.10} & \textbf{0.076} & \textbf{0.10} & \underline{1.39} & \textbf{2.60} & \textbf{0.200} & \textbf{-0.11} & \textbf{1.80} & \textbf{3.60} & \textbf{0.100} & \underline{0.10} \\
\quad w/o $\mathcal{L}_{\text{deg}}$ & \underline{1.81} & \underline{5.07} & 0.032 & \underline{0.17} & \textbf{1.40} & 2.55 & \underline{0.181} & \underline{-0.03} & 1.47 & 3.05 & 0.162 & \textbf{-0.05} \\
\quad w/o $\mathcal{L}_{\text{tri}}$ & 1.77 & 4.81 & \underline{0.120} & 0.34 & \underline{1.39} & \underline{2.58} & 0.177 & -0.02 & \underline{1.66} & \underline{3.16} & \underline{0.102} & 0.13 \\
\bottomrule
\end{tabular}
\end{adjustbox}
\end{table*}
\subsection{Ablation study}\label{sec:ablation}
Table~\ref{tab:ablation_consolidated} isolates the contributions of our prior scaffold and regularizers. Extended derivations, hyperparameters, and inference step-reduction analyses are provided in Appendices~\ref{appx:step_reduction}--\ref{appx:regularizers}. In addition, we validate this SNR improvement empirically in Appendix~\ref{app:practical}.
\paragraph{Terminal Prior Alignment.} 
We test whether the terminal scaffold prior must match the inductive bias of our degree-aware masking schedule ($\beta^{ij}$). Replacing our degree-based Chung--Lu prior with flat Erd\H{o}s--R\'enyi (ER) or a class-based SBM damages topological fidelity across datasets lacking explicit community structure: on Alpha, ER and SBM degrade clustering coefficient (CC) by $5\times$ to $6\times$. SBM only remains competitive on Polblogs, whose binary ideological communities explicitly fit the block-model prior. This confirms that matching the prior to the degree schedule is essential.
\paragraph{Role of Structural Losses.} 
Removing $\mathcal{L}_{\text{deg}}$ systematically deteriorates CC (worsening error by $2\text{--}4\times$), as CC is normalized by node degree and therefore distorted when degree distributions shift. Conversely, ablating $\mathcal{L}_{\text{tri}}$ primarily damages degree assortativity (AC error nearly doubles on Cora and Polblogs). Meanwhile, PLE remains stable under both ablations ($|\Delta_{\text{PLE}}| \le 0.02$), confirming that the global degree distribution is anchored by the forward schedule $\beta^{ij}$, while $\mathcal{L}_{\text{deg}}$ and $\mathcal{L}_{\text{tri}}$ provide localized topological corrections.
\paragraph{Inference Efficiency.}
Since our reverse posterior (Eq.~\ref{eq:unmask}) is well-defined for
arbitrary step skips, we can reduce the number of inference steps without
modifying the sampling procedure. Empirically
(Table~\ref{tab:step_reduction_grouped}, Fig.~\ref{fig:inference_time}),
inference steps can be reduced to $50\%$ of the training schedule
($T{=}50\%$) with negligible metric degradation, cutting sampling time in
half; performance drops noticeably only at extreme reductions ($T{=}10\%$).
Complete curves across datasets appear in Appendix~\ref{appx:step_reduction}.
\section{Conclusion}
We introduced a discrete diffusion framework for graph generation that
restricts training to a structurally motivated subset of node pairs,
together with a small random coverage sample, yielding an unbiased
training objective with sub-quadratic training complexity in the number
of nodes. A degree-aware forward process keeps the noisy graph
structurally informative throughout the reverse process, while a dynamic
common-neighbor feature conditions the denoiser on local structure at
every noise level. Together, these design choices improve structural
fidelity over existing discrete diffusion baselines without memorizing the
training graphs. Across datasets of varying size and domain, our method is
designed for, and achieves the best or second-best fidelity on, large
graphs, while remaining comparable to baselines built specifically for
small or molecular graphs.

\newpage
\bibliography{iclr2027_conference.bib}

\newpage
\newpage
\appendix
\section*{Appendix Overview}

\startcontents[appendix]
\printcontents[appendix]{l}{1}{\section*{Contents}\hrule\vspace{0.8em}}
\vspace{1.5em}
\hrule
\vspace{2em}

\newpage
\section{Training and Sampling Algorithms}\label{appx:training_alg}
Here we provide the pre-computation~\ref{alg:ergm_precompute}, optimization during training~\ref{alg:ergm_ddpm_short}, graph initialization during sampling~\ref{alg:ergm_sample_init} and sampling steps~\ref{alg:ergm_sample_step} algorithms. 

\begin{algorithm}[h]
\footnotesize
\caption{Pre-computation (once per graph dataset)}
\label{alg:ergm_precompute}
\begin{algorithmic}[1]
\Require $G=(X_0, E_0) \sim  Dataset$
\State Compute degree sequence $\mathbf d^{*}$, triangle count $T^{*}$, and the candidate set $\calQ = E_0\cup\calW(G)$\Comment{Eq.~\ref{eq:Q}}
\State Compute masking exponents $\{\beta^{ij}\}$, HT weights $\{w^{ij}\}$, and Chung-Lu parameters $\{\rho^{ij}\}$ for $(i,j)\in\calQ$ \Comment{Eqs.~\ref{eq:beta},\ref{eq:LT}}
\State $L_T \leftarrow$ prior term (Eq.~\ref{eq:LT}) \Comment{$\theta$- and $t$-independent scalar, cached once}
\State \Return $\calQ,\ \mathbf d^*,\ T^*,\ \{\beta^{ij}\},\ \{w^{ij}\},\ \{\rho^{ij}\},\ L_T$
\end{algorithmic}
\end{algorithm}

\begin{algorithm}[h]
\footnotesize
\caption{Optimization step of $\mathcal{L}(\theta)$, Eq.~\eqref{eq:elbo_full}}
\label{alg:ergm_ddpm_short}
\begin{algorithmic}[1]
\Require $\calQ,\mathbf d^*,T^*,\{\beta^{ij}\},\{w^{ij}\},L_T$ from Alg.~\ref{alg:ergm_precompute}; model $f_\theta=(h_\theta,g_\theta)$
\State $t$ $\sim$ \{1, ..., T\} 
\State Sample coverage set $S_t\subseteq\calQ^c$ of size $\lceil\lambda_S\,|\calQ^c|\rceil$; set $\calQ_t\gets\calQ\cup S_t$ \Comment{Eq.~\ref{eq:Qt}}

\State \(G_t\) \(\gets\) \(q(G_t|G_0)\)
\State Compute $\{\tau^{ij}_t\}$ on $\calQ_t$ (Eq.~\ref{eq:quantities})
\State $\mathbf H_t \gets h_\theta(G_t)$ \Comment{sparse encoder, message passing restricted to visible edges $E_t$, Eq.~\ref{eq:encoder}}
\State $\hat p_{ij} \gets g_\theta(h^i_t,h^j_t,\log(1+\tau^{ij}_t))$ for $(i,j)\in\calQ_t$ \Comment{pairwise decoder, Eq.~\ref{eq:decoder}}
\State Compute the class-balance weight $w^+_b(t)$ \Comment{Eq.~\ref{eq:wpos}} 
\State Compute \(\mathcal{L}_\theta\)\Comment{Eq.~\ref{eq:elbo_full}}
\State \(\mathcal{L_\theta}.\text{backward()}\)
\end{algorithmic}
\end{algorithm}

\paragraph{Sampling.}
Since $\beta^{ij}$, the encoder, and the prior all require node degrees,
generation is conditioned on a degree sequence: we draw a training graph
uniformly at random and take its size $N$ and degree sequence
$\mathbf d^*=\{d^{i*}\}$, which stay fixed for the whole reverse chain and
replace the true degrees in Eqs.~\ref{eq:beta}, \ref{eq:encoder}, and
\ref{eq:LT}. The edges are generated by the model, so $\mathbf d^*$ acts
as a target that the generated graph follows in expectation rather than
exactly. The chain starts from the partial-absorbing Chung--Lu prior,
drawn independently for every pair $(i,j)$, $i<j$:
\begin{equation}
    p\bigl(e^{(i,j)}_T=\M\bigr)=1-\alpha_{\min}^{\beta^{ij}},\qquad
    p\bigl(e^{(i,j)}_T=1\bigr)=\alpha_{\min}^{\beta^{ij}}\,\rho^{ij},\qquad
    p\bigl(e^{(i,j)}_T=0\bigr)=\alpha_{\min}^{\beta^{ij}}\,\bigl(1-\rho^{ij}\bigr),
    \label{eq:prior}
\end{equation}
with $\rho^{ij}$ as in Eq.~\ref{eq:LT}. Each pair is thus visible with the
same probability as under the forward marginal at $t=T$
(Eq.~\ref{eq:fwd-retain}), and its visible state is drawn from Chung--Lu;
accordingly, the KL between the forward marginal at $T$ and this prior is,
per pair, $\alpha_{\min}^{\beta^{ij}}\,\CE\bigl(e^{(i,j)}_0,\rho^{ij}\bigr)$,
the summand of $\mathcal L_{\rm prior}$. We then run the reverse process
along a decreasing step sequence $T=t_K>t_{K-1}>\dots>t_0=0$, $K\le T$,
which Eq.~\ref{eq:unmask} permits for any $t_{\rm prev}<t$. At each step,
$\tau^{ij}_t$ and $\mathbf d_t$ are recomputed from the visible edges
$E_t$, and every masked pair is revealed with probability
$r^{ij}_{t\to t_{\rm prev}}$ and then set to
$\tilde e^{ij}_0\sim\mathrm{Bernoulli}(\hat p_{ij})$, while visible pairs
keep their state (Eq.~\ref{eq:posterior}). At the last step,
$r^{ij}_{t\to 0}=1$ decodes all remaining masked pairs, and the generated
graph has edge set $\{(i,j):e^{(i,j)}_0=1\}$. Unlike training, sampling
cannot restrict the queries to $\calQ$, which is defined by the unknown
clean graph. A dynamic set built from the visible edges would miss hub
pairs, which have $\beta^{ij}$ close to $\beta_{\max}$, are almost surely
masked at $t=T$, and rarely share a visible neighbor early in the chain.
The denoiser therefore scores all $\binom N2$ pairs at every step,
evaluated in chunks (Algs.~\ref{alg:ergm_sample_init}
and~\ref{alg:ergm_sample_step}).

\begin{algorithm}[h]
\footnotesize
\caption{Sampling initialization ($t=T$, once per generated graph)}
\label{alg:ergm_sample_init}
\begin{algorithmic}[1]
\State Sample graph size $N$ and degree sequence $\mathbf d^*$ from the empirical training distribution
\State Compute $\{\beta^{ij}\}$ and $\{\rho^{ij}\}$ for all $\binom{N}{2}$ pairs (Eqs.~\ref{eq:beta},\ref{eq:LT})
\State Sample $E_T$: each pair is visible w.p.\ $\alpha_{\min}^{\beta^{ij}}$, with visible edges drawn Chung-Lu Bernoulli($\rho^{ij}$); all other pairs are masked
\State \Return $E_T,\ \{\beta^{ij}\},\ \{\rho^{ij}\}$
\end{algorithmic}
\end{algorithm}

\begin{algorithm}[h]
\footnotesize
\caption{Reverse step $t\to t_{\rm prev}$}
\label{alg:ergm_sample_step}
\begin{algorithmic}[1]
\Require current state $E_t$; $\{\beta^{ij}\}$; model $f_\theta=(h_\theta,g_\theta)$
\State compute $\{\tau^{ij}_t\}$ for all pairs
\State $\mathbf H_t \gets h_\theta(G_t)$
\State $\hat p_{ij} \gets g_\theta(\mathbf{h}_i,\mathbf{h}_j,\log(1+\tau^{ij}_t))$ for all pairs \Comment{query set is every pair, since sampling must decide each pair's final state.}

\For{each $(i,j)\in\calQ_t$ with $e^{ij}_{t}=\M$}
        \State Compute $r_{t\to t-1}^{ij}$ via Eq.\ref{eq:unmask}
        \State With prob.\ $r_{t\to t-1}^{ij}$:
               $e^{ij}_{t-1}\sim\mathrm{Categorical}(\hat{p}_{ij})$;
               else $e^{ij}_{t-1}\leftarrow\M$
    \EndFor
\State \Return $E_{t_{\rm prev}}$ \Comment{after $t-1=-1$: output edge set $=\{(i,j):e^{(i,j)}_0=1\}$}
\end{algorithmic}
\end{algorithm}

\section{Analysis and Discussions}\label{appx:Analysis}
As discussed above, by controlling the size of the randomly sampled coverage set $S_t$
(Eq.~\ref{eq:lambdaS_choice}), our model operates on a sub-quadratic (in $N$) fraction
of all possible node pairs during training — $O(N\mu_2)$ rather than $O(N^2)$ —
enabling scalability to large sparse graphs. During sampling, by contrast, the model
must evaluate the state of every potential edge at each reverse step, since the reverse
process has no access to the ground-truth candidate set $\calQ$ and must therefore
resolve all $\binom N2$ pairs; this introduces an unavoidable quadratic dependency on
the number of nodes at sampling time (Table~\ref{tab:complexity_diffusion_only}).

As shown in Tables~\ref{tab:complexity_diffusion_only} and
\ref{tab:complexity_nondiffusion}, our training complexity compares favorably with
EDGE~\citep{EDGE} in a specific sense: our bound $O(N\mu_2)$ depends only on
static graph structure and holds unconditionally by construction, whereas EDGE's bound
$O(\max(m,K^2))$ depends on the number of ``active" nodes $K$ at each step, a
stochastic, schedule-dependent quantity rather than a proven general bound. In this
sense, our training-time efficiency is a guaranteed worst-case property of the graph,
while EDGE's is an empirical regularity of its noise schedule. The comparison reverses
at sampling time, however: EDGE's active-node restriction applies to its reverse
process as well, plausibly preserving subquadratic sampling cost (to the extent its
empirical $K\ll N$ observation holds), while our sampling procedure — lacking an
analogous restriction — incurs the full $O(N^2)$ cost of dense diffusion baselines
such as DiGress at every reverse step. This training/sampling asymmetry is a genuine limitation of the current method
relative to EDGE. The asymmetry follows
from a structural difference in convergent distributions: EDGE's forward process
converges to the literal empty graph $G(N,0)$, which makes ``active" (currently
changing) nodes well-defined and lets the reverse process monotonically grow the
graph from nothing, decoding only those nodes at each step. Our forward process
instead uses a \emph{partial} absorbing prior, $\alpha_{\min}\in(0,1)$
(Eq.~\ref{eq:schedule}), which does not converge to a fully empty or fully masked
state; this choice is motivated by avoiding the SNR collapse of sparse graphs at the
absorbing limit and by training convergence
considerations, but it forfeits the empty-graph-anchored monotonic restriction that
EDGE exploits. Recovering a comparable sampling-time restriction under a partial
absorbing prior — e.g., restricting decoding to the currently masked pool, which
shrinks monotonically over the reverse schedule even though it starts non-empty — is
a natural direction for future work, left outside the scope of the current
formulation.

\begin{table}[h]
\centering
\small
\begin{threeparttable}
\caption{Per-step complexity, diffusion-based graph generative models only. Assumes constant per-call network cost, so candidate-pair count and per-step training complexity coincide. $N$ is number of nodes and $m$ is number of edges. Bold denotes complexity improved by our method relative to the diffusion baselines.}
\label{tab:complexity_diffusion_only}
\begin{tabular}{@{}l p{3.6cm} p{4.2cm}@{}}
\toprule
\textbf{Method} & \textbf{Training (per step)} & \textbf{Sampling (per step)} \\
\midrule
DiGress            & $O(N^2)$ & $O(N^2)$ \\
SparseDiff           & $O(m)$\tnote{a} & $O(N^2)$\tnote{b} \\
EDGE                & $O(\max(m,K^2))$\tnote{c} & $O(\max(m,K^2))$\tnote{c} \\
\textbf{Ours} & $\bm{O(N\mu_2)}$\tnote{d} & $O(N^2)$ \\
\bottomrule
\end{tabular}
\begin{tablenotes}[flushleft]
\footnotesize
\item[a] By explicit design choice $\lambda=m/N^2$ \citep{SparsDiff}.

\item[b] Chunked for memory only; total compute per step still touches all $\binom{N}{2}$ pairs.

\item[c] \citet{EDGE}: the per-step term shown here is the summand already implicit in their stated total $O(T\max(K^2,m))$ over $T$ steps (i.e., total $= T \times$ per-step). $K$ = number of ``active" nodes.

\item[d] $\mu_2 = \frac{1}{N}\sum_v d_v^2$, the mean squared degree; by the explicit choice of $\lambda_S$ (Eq.~\ref{eq:lambdaS_choice}).
\end{tablenotes}
\end{threeparttable}
\end{table}

\begin{table*}[h]
\centering
\small
\begin{threeparttable}
\caption{Comparison of autoregressive (AR) graph generation methods. Complexity values are adopted from the original papers for consistency. $N$ is number of nodes and $m$ is number of edges. Note that ``per-step'' here denotes one autoregressive generation step (row/node/token), whose number scales with $N$ or $m$, which is not directly comparable to the fixed-$T$ diffusion timestep used as ``per-step'' in Table~\ref{tab:complexity_diffusion_only}.}
\label{tab:complexity_nondiffusion}
\begin{tabular}{@{}l l p{2.6cm} p{3.2cm} p{3.2cm}@{}}
\toprule
\textbf{Method} & \textbf{Paradigm} & \textbf{Generation Process} & \textbf{Overall Complexity} & \textbf{Per-step Complexity} \\
\midrule
GraphRNN & Flat AR & Row-by-row  & $O(N^2)$\tnote{a} & $O(N)$ \\
BiGG & Flat AR (tree) & Tree traversal & $O((N+m)\log N)$ & $O(1)$ \\
HiGen & Hierarchical AR & Cluster-wise & $O(n_c\log^2 N)$\tnote{b} & $O(N^2)$\tnote{b} \\
HGGT ($K^2$-trees) & Compressed AR & Sequence  & $O(L_{\rm seq}^2)$\tnote{c} & $O(L_{\rm seq}^2)$\tnote{c} \\
\bottomrule
\end{tabular}

\begin{tablenotes}[flushleft]
\footnotesize
\item[a] $N$ row-emission steps $\times$ $O(N)$ per step.

\item[b] \citet{Higen}. $n_c$ = number of clusters. $O(N^2)$ full attention or $O(N+m)$ linear attention for larger graphs; Per-step and Overall coincide with the full-attention variant here.

\item[c] $L_{\rm seq}=O(m\log(N^2/m))$ is the flattened $K^2$-tree sequence length. Per-step and Overall coincide since generation is a single-pass sequence operation rather than one cost repeated over $N$ or $m$ steps.

\end{tablenotes}
\end{threeparttable}
\end{table*}
\paragraph{Dependence on degree heterogeneity.}
Unlike SparsDiff, which sets $\lambda := m/N^2$ so its query count
$\lambda N^2=m$ is $O(m)$ unconditionally, our per-step cost
$O(m+N\mu_2)$ depends on the second moment of the degree distribution,
$\mu_2 := \frac{1}{N}\sum_v d_v^2$, since $\calW(G)$ is enumerated
deterministically rather than sampled. For each node $v$, computing its
wedges costs $O(d_v^2)$, giving
$|\calW(G)| = O\!\big(\sum_v d_v^2\big) = O(N\mu_2)$ in total. This term
reduces to $O(m)$ exactly when the condition
\begin{equation}
    \mu_2 \;=\; O(\bar d), \qquad \bar d := \frac{2m}{N},
    \label{eq:mu2-condition}
\end{equation}
holds, i.e., when the second moment of the degree distribution is of the
same order as the mean degree. A sufficient condition for
Eq.~\ref{eq:mu2-condition} to hold is a bounded maximum degree,
$\Delta = O(1)$, since then $\mu_2 = \frac{1}{N}\sum_v d_v^2 \le \Delta\bar d
= O(\bar d)$. This holds for many sparse real-world graphs but not all.
A star graph, for instance, has $m=N-1=O(N)$ edges, yet
$|\calW(G)|=\binom{N-1}{2}=O(N^2)$, since every leaf pair shares the hub as
a common neighbor. This shows that sparsity in $m$ alone does not
guarantee $O(m)$ cost under high degree heterogeneity. Accordingly, since
$2m \le N\mu_2$ always holds (as $d_v^2 \ge d_v$ for integer degrees), we report
our complexity as $O(N\mu_2)$ throughout
(Table~\ref{tab:complexity_diffusion_only}): this is the exact asymptotic
order in all cases -- including those where it coincides with $O(m)$ --
and one that scales with the structural complexity that drives the cost
of deterministic wedge enumeration, rather than with $m$ alone.
\paragraph{Choice of $\lambda_S$.} To guarantee the coverage sample does not
reintroduce quadratic cost, and is not an arbitrary hyperparameter, we fix
\begin{equation}
\lambda_S \;:=\; \frac{c\,N\mu_2}{\binom{N}{2}-m-|\calW(G)|}, \qquad \mu_2 = \frac{1}{N}\sum_{v} d_v^2,\ \ c>0 \text{ a constant},
\label{eq:lambdaS_choice}
\end{equation}
so that $|S_t| = \bigl\lceil \lambda_S\bigl(\tbinom N2 - m - |\calW(G)|\bigr)\bigr\rceil = O(N\mu_2)$ by construction — matching the scale of $|\calQ|=O(N\mu_2)$ rather than being an independent hyperparameter. This mirrors SparseDiff's choice $\lambda=m/N^2$ \citep{SparsDiff}, which analogously pins their query-edge count to $O(m)$. Under Eq.~\ref{eq:lambdaS_choice}, training's per-step complexity (Table~\ref{tab:complexity_diffusion_only}) is $O(N\mu_2)$, not merely bounded in expectation but by construction. In practice, we select $c$ per dataset to hold the realized per-step tensor
footprint, $\text{batch}\times\text{dim}\times(|\calQ|+|S_t|)$, within a fixed
memory budget: $c$ decreases as $\mu_2$ (and hence $|\calQ|=m+|\calW(G)|$)
grows, since a larger $\calQ$ leaves less headroom for $S_t$ before training
reintroduces quadratic-scale cost or exhausts GPU memory. This is a
complexity/memory safety margin, not a quality-tuning knob — $c$ is chosen to
be the largest value that keeps $|\calQ|+|S_t|$ within budget, not selected
by any generation-quality criterion.

\begin{lemma}[Wedge Count Bound]
\label{lem:wedge-bound}
Let $G=(V,E)$ be a graph on $N$ nodes with edge count $m=|E|$, degree
sequence $\{d_v\}_{v\in V}$, and second moment
$\mu_2 = \frac{1}{N}\sum_{v} d_v^2$. Then
\begin{equation}
    |\calW(G)| \;\le\; \frac{1}{2}\bigl(N\mu_2 - 2m\bigr).
    \label{eq:wedge-bound}
\end{equation}
\end{lemma}

\begin{proof}
For a fixed node $k\in V$, the number of unordered node pairs $(i,j)$ with
$k$ as a common neighbour is $\binom{d_k}{2}$, since this counts the ways to
choose two distinct neighbours of $k$ from its neighbourhood of size $d_k$.
Summing over all nodes $k$ counts each pair $(i,j)$ once for every common
neighbour it has, giving the double-counting identity
\begin{equation}
    \sum_{k=1}^{N}\binom{d_k}{2}
    \;=\;
    \sum_{(i,j)}\bigl|\mathcal N(i)\cap\mathcal N(j)\bigr|.
    \label{eq:double-count}
\end{equation}
Restricting the right-hand sum in \eqref{eq:double-count} to non-edges only
removes non-negative terms, and replacing each remaining count
$|\mathcal N(i)\cap\mathcal N(j)|\ge 1$ by the indicator $\mathbf 1[\cdot\ge1]$
only decreases each term further. Hence
\begin{equation}
    \sum_{k=1}^{N}\binom{d_k}{2}
    \;\ge\;
    \sum_{(i,j)\notin E}
    \mathbf 1\bigl[|\mathcal N(i)\cap\mathcal N(j)|\ge 1\bigr]
    \;=\;
    |\calW(G)|,
    \label{eq:wedge-ineq}
\end{equation}
where the final equality is Definition~\ref{def:W} of $\calW(G)$. Expanding
the left-hand side of \eqref{eq:wedge-ineq},
\begin{equation}
    \sum_{k=1}^N \binom{d_k}{2}
    \;=\;
    \frac{1}{2}\sum_{k=1}^N d_k(d_k-1)
    \;=\;
    \frac{1}{2}\Bigl(\sum_{k=1}^N d_k^2 - \sum_{k=1}^N d_k\Bigr)
    \;=\;
    \frac{1}{2}\bigl(N\mu_2 - 2m\bigr),
    \label{eq:expand}
\end{equation}
using $\sum_k d_k^2 = N\mu_2$ and the handshake identity
$\sum_k d_k = 2m$. Combining \eqref{eq:wedge-ineq} and \eqref{eq:expand}
gives \eqref{eq:wedge-bound}.
\end{proof}

\begin{proposition}[Coverage Bound]
\label{prop:coverage}
Let $G$ be a graph on $N$ nodes with $m$ edges and degree second moment
$\mu_2 = \frac{1}{N}\sum_v d_v^2$, and let $\calQ_t = E\cup\calW(G)\cup S_t$
be the training-time candidate set of Eq.~\ref{eq:Qt}, with $|S_t| =
\lceil cN\mu_2\rceil$ fixed by the choice of $\lambda_S$
(Eq.~\ref{eq:lambdaS_choice}) for a constant $c>0$. Then
\begin{equation}
    |\calQ_t| \;\le\; \Bigl(\tfrac12+c\Bigr)N\mu_2 + 1
    \;=\; O(N\mu_2),
    \label{eq:coverage-bound}
\end{equation}
and consequently the fraction of all $\binom{N}{2}$ node pairs evaluated at
each training step satisfies
\begin{equation}
    \frac{|\calQ_t|}{\binom{N}{2}} \;=\; O\!\left(\frac{\mu_2}{N}\right),
    \label{eq:coverage-fraction}
\end{equation}
which vanishes as $N\to\infty$ whenever $\mu_2 = O(N)$.
\end{proposition}

\begin{proof}
By Definition~\ref{def:W}, $\calQ = E\cup\calW(G)$ with $E\cap\calW(G)=
\emptyset$, so $|\calQ| = m + |\calW(G)|$. Applying
Lemma~\ref{lem:wedge-bound},
\begin{equation}
    |\calQ| \;=\; m + |\calW(G)|
    \;\le\; m + \frac{1}{2}\bigl(N\mu_2 - 2m\bigr)
    \;=\; \frac{1}{2}N\mu_2.
    \label{eq:Q-bound}
\end{equation}
By construction, $|S_t| = \lceil cN\mu_2\rceil \le cN\mu_2 + 1$ holds
deterministically for every realization of $S_t$, since $\lambda_S$ fixes
the sample \emph{size}, not merely its expectation. Since $\calQ_t =
\calQ \cup S_t$ with $S_t\subseteq \bar E\setminus\calW(G)$ disjoint from
$\calQ$ (Eq.~\ref{eq:Qt}),
\begin{equation}
    |\calQ_t| \;=\; |\calQ| + |S_t|
    \;\le\; \frac{1}{2}N\mu_2 + cN\mu_2 + 1
    \;=\; \Bigl(\tfrac12+c\Bigr)N\mu_2 + 1,
\end{equation}
which is \eqref{eq:coverage-bound}. Dividing by $\binom{N}{2} =
\Theta(N^2)$ gives
\begin{equation}
    \frac{|\calQ_t|}{\binom{N}{2}}
    \;=\; O\!\left(\frac{N\mu_2}{N^2}\right)
    \;=\; O\!\left(\frac{\mu_2}{N}\right),
\end{equation}
which is \eqref{eq:coverage-fraction}. This ratio tends to $0$ as
$N\to\infty$ precisely when $\mu_2 = O(N)$.
\end{proof}

\begin{proposition}[Unbiasedness of the Reweighted ELBO]
\label{prop:unbiased}
Fix a diffusion step $t\in\{1,\dots,T-1\}$ and a clean graph $\mathbf e_0$.
Let
\begin{equation}
    \hat L(t)
    \;=\;
    \underbrace{\sum_{(i,j)\in E}\ell_{ij}(t)}_{\text{edge term}}
    \;+\;
    \underbrace{\sum_{(i,j)\in\calW(G)}\ell_{ij}(t)}_{\text{wedge term}}
    \;+\;
    \underbrace{\frac{1}{\lambda_S}\sum_{(i,j)\in S_t}\ell_{ij}(t)}_{\text{coverage term}},
    \label{eq:Lhat}
\end{equation}
where $\ell_{ij}(t) := w^{ij}_{\rm pos}(t)\,r_t^{ij}\,
\mathbf 1[e^{(i,j)}_t{=}\M]\,\CE\bigl(e^{(i,j)}_0,\hat p_{ij}\bigr)$, $S_t$ is
drawn uniformly without replacement from $\bar E\setminus\calW(G)$ with
$|S_t|=\lambda_S\bigl(\binom{N}{2}-m-|\calW(G)|\bigr)$
(Eq.~\ref{eq:lambda}), and $E$, $\calW(G)$ are deterministic given $G$.
Let
\begin{equation}
    L_{\rm full}(t) \;:=\; \sum_{(i,j)}\ell_{ij}(t)
\end{equation}
denote the same loss summed over all $\binom{N}{2}$ node pairs. Then
$\hat L(t)$ is an unbiased estimator of $L_{\rm full}(t)$ with respect to
the randomness in $S_t$:
\begin{equation}
    \mathbb E_{S_t}\bigl[\hat L(t)\bigr] \;=\; L_{\rm full}(t).
    \label{eq:unbiased-claim}
\end{equation}
\end{proposition}

\begin{proof}
The set of all $\binom{N}{2}$ node pairs decomposes into three disjoint
parts:
\begin{equation}
    \{(i,j)\} \;=\; E \;\cup\; \calW(G) \;\cup\; \bigl(\bar E\setminus\calW(G)\bigr),
    \label{eq:partition}
\end{equation}
so $L_{\rm full}(t)$ splits additively into the corresponding three sums,
and by linearity it suffices to match each term of $\hat L(t)$ in
\eqref{eq:Lhat} to the corresponding part of $L_{\rm full}(t)$ in
expectation.

\textbf{Edge term.} $E$ is fixed given $G$ and does not depend on $S_t$, so
\begin{equation}
    \mathbb E_{S_t}\Bigl[\sum_{E}\ell_{ij}(t)\Bigr]
    \;=\; \sum_{E}\ell_{ij}(t),
\end{equation}
which equals the edge part of $L_{\rm full}(t)$ exactly.

\textbf{Wedge term.} Likewise, $\calW(G)$ is fixed given $G$
(Definition~\ref{def:W}) and does not depend on $S_t$, so
\begin{equation}
    \mathbb E_{S_t}\Bigl[\sum_{\calW(G)}\ell_{ij}(t)\Bigr]
    \;=\; \sum_{\calW(G)}\ell_{ij}(t),
\end{equation}
which equals the wedge part of $L_{\rm full}(t)$ exactly.

\textbf{Coverage term.} $S_t$ is drawn uniformly without replacement from
the population $\bar E\setminus\calW(G)$, so every pair in this population
has identical inclusion probability:
\begin{equation}
    \Pr\bigl[(i,j)\in S_t\bigr] \;=\; \frac{|S_t|}{|\bar E\setminus\calW(G)|} \;=\; \lambda_S,
    \qquad \forall\,(i,j)\in\bar E\setminus\calW(G).
    \label{eq:inclusion-prob}
\end{equation}
Writing the coverage sum via inclusion indicators
$\mathbf 1[(i,j)\in S_t]$ and applying linearity of expectation together
with \eqref{eq:inclusion-prob},
\begin{align}
    \mathbb E_{S_t}\Bigl[\sum_{S_t}\ell_{ij}(t)\Bigr]
    &= \mathbb E_{S_t}\Bigl[\sum_{(i,j)\in\bar E\setminus\calW(G)}
       \mathbf 1[(i,j)\in S_t]\,\ell_{ij}(t)\Bigr] \\
    &= \sum_{(i,j)\in\bar E\setminus\calW(G)}
       \Pr\bigl[(i,j)\in S_t\bigr]\,\ell_{ij}(t)
    \;=\; \lambda_S \!\!\sum_{\bar E\setminus\calW(G)}\!\! \ell_{ij}(t).
\end{align}
Hence the Horvitz--Thompson-corrected coverage term satisfies
\begin{equation}
    \mathbb E_{S_t}\Bigl[\frac{1}{\lambda_S}\sum_{S_t}\ell_{ij}(t)\Bigr]
    \;=\; \frac{1}{\lambda_S}\cdot\lambda_S\!\!\sum_{\bar E\setminus\calW(G)}\!\!\ell_{ij}(t)
    \;=\; \sum_{\bar E\setminus\calW(G)}\ell_{ij}(t),
\end{equation}
which equals the remaining part of $L_{\rm full}(t)$ exactly.

Summing the three matched terms and using the partition
\eqref{eq:partition} gives $\mathbb E_{S_t}[\hat L(t)] = L_{\rm full}(t)$,
proving \eqref{eq:unbiased-claim}.
\end{proof}

\begin{proposition}[Bias of the $\lambda$-Weighted Wedge Term]
\label{prop:bias}
Suppose the wedge term of the ELBO is weighted by $1/\lambda$, i.e.
\begin{equation}
    \hat L_\lambda(t)
    \;=\;
    \sum_{(i,j)\in E}\ell_{ij}(t)
    \;+\;
    \frac{1}{\lambda}\sum_{(i,j)\in\calW(G)}\ell_{ij}(t)
    \;+\;
    \frac{1}{\lambda_S}\sum_{(i,j)\in S_t}\ell_{ij}(t),
    \label{eq:Lhat-lambda}
\end{equation}
with $\lambda = |\calW(G)|/(\binom{N}{2}-m)$ as in Eq.~\ref{eq:lambda}, and
let $L_{\rm full}(t)$ denote the full-graph loss as in
Proposition~\ref{prop:unbiased}. Then, with respect to the randomness in
$S_t$,
\begin{equation}
    \mathbb{E}_{S_t}\bigl[\hat L_\lambda(t)\bigr]
    \;=\;
    L_{\rm full}(t) \;+\; \mathrm{Bias}(t),
    \qquad
    \mathrm{Bias}(t)
    \;=\;
    \frac{1-\lambda}{\lambda}
    \sum_{(i,j)\in\calW(G)}\ell_{ij}(t).
    \label{eq:bias-def}
\end{equation}
\end{proposition}

\begin{proof}
The edge and coverage terms of $\hat L_\lambda(t)$ are identical to
$\hat L(t)$ in Proposition~\ref{prop:unbiased}, so their expectations match
the corresponding parts of $L_{\rm full}(t)$ exactly, as shown there. Only
the wedge term differs, by the constant factor $1/\lambda$:
\begin{equation}
    \mathbb E_{S_t}\Bigl[\frac{1}{\lambda}\sum_{\calW(G)}\ell_{ij}(t)\Bigr]
    \;=\;
    \frac{1}{\lambda}\sum_{\calW(G)}\ell_{ij}(t)
    \;=\;
    \sum_{\calW(G)}\ell_{ij}(t)
    \;+\;
    \left(\frac{1}{\lambda}-1\right)\sum_{\calW(G)}\ell_{ij}(t),
\end{equation}
since $\calW(G)$ is deterministic and does not depend on $S_t$. The first
term on the right matches the wedge part of $L_{\rm full}(t)$ exactly; the
second term is the excess, giving
$\mathrm{Bias}(t) = \bigl(\tfrac1\lambda-1\bigr)\sum_{\calW(G)}\ell_{ij}(t)
= \tfrac{1-\lambda}{\lambda}\sum_{\calW(G)}\ell_{ij}(t)$. Summing all three
terms and comparing to $L_{\rm full}(t)$ gives \eqref{eq:bias-def}.
\end{proof}

\begin{remark}[Order of the Bias]
\label{rem:bias-order}
Write $\bar\ell_{\calW}(t) := \frac{1}{|\calW(G)|}\sum_{\calW(G)}\ell_{ij}(t)$
for the average per-pair wedge loss at step $t$ (assuming $\calW(G)\ne
\emptyset$). By definition of $\lambda$,
\begin{equation}
    \frac{1}{\lambda}\,|\calW(G)|
    \;=\;
    \binom{N}{2}-m,
\end{equation}
so the bias in \eqref{eq:bias-def} simplifies exactly to
\begin{equation}
    \mathrm{Bias}(t)
    \;=\;
    (1-\lambda)\Bigl(\binom{N}{2}-m\Bigr)\,\bar\ell_{\calW}(t).
    \label{eq:bias-simplified}
\end{equation}
By Proposition~\ref{prop:coverage}, $\lambda = O(\mu_2/N) \to 0$ as
$N\to\infty$ whenever $\mu_2=o(N)$ (the same regime in which coverage stays
sub-quadratic), so $(1-\lambda)\to 1$ and
\begin{equation}
    \mathrm{Bias}(t) \;=\; \Theta\!\bigl(N^2\bigr)\cdot\bar\ell_{\calW}(t)
    \qquad\text{as } N\to\infty.
    \label{eq:bias-order}
\end{equation}
\end{remark}
This is the \emph{same asymptotic order}, $\Theta(N^2)$, as
$L_{\rm full}(t)$'s own coverage-population part
$\sum_{\bar E\setminus\calW(G)}\ell_{ij}(t)$, which sums over
$\binom N2-m-|\calW(G)| = \Theta(N^2)$ terms. Consequently, the bias is
negligible relative to $L_{\rm full}(t)$ \emph{only if}
$\bar\ell_{\calW}(t)$ is asymptotically smaller than the typical
non-wedge, non-edge loss — a condition not implied by anything already
established, and implausible in practice: wedge pairs are, by construction,
the structurally closest non-edges to becoming edges, so a well-trained
denoising network is expected to assign them \emph{comparable or larger}
loss than generic non-edges, not smaller. Absent such a separate,
justified assumption, $\mathrm{Bias}(t)$ is $\Theta(N^2)$ and therefore
\textbf{not negligible} — it is of the same order as the full-graph loss
itself, not a vanishing correction to it.

\begin{definition}[Class-Mass SNR]
\label{def:snr}
For a training step $t$ and pool weighting $w^{ij}$, define the
\emph{class-mass signal-to-noise ratio} as the ratio of total weighted mass
assigned to true edges versus non-edges within the current masked pool
(Eq.~\eqref{eq:pool_indicator}):
\begin{equation}
    \mathrm{SNR}_w(t)
    \;:=\;
    \frac{\displaystyle\sum_{(i,j)\in\calQ_t} w^{ij}\,e^{(i,j)}_0\,Q_{ij}(t)}
         {\displaystyle\sum_{(i,j)\in\calQ_t} w^{ij}\,\bigl(1-e^{(i,j)}_0\bigr)\,Q_{ij}(t)}.
    \label{eq:snr-def}
\end{equation}
For the unweighted baseline ELBO ($w^{ij}\equiv 1$ for all pairs),
\eqref{eq:snr-def} reduces to $\mathrm{SNR}_1(t) = n^+_t/n^-_t$, using
$n^+_t$, $n^-_t$ as defined in Eq.~\ref{eq:npos_nneg}.
\end{definition}

\begin{proposition}[ Exact SNR Balancing]
\label{prop:snr}
Let $w^{ij}_{\rm pos}(t)$ be defined as in Eq.~\ref{eq:wpos}, and assume
$n^+_t\ge 1$. Then the reweighted pool is exactly class-balanced at every
step:
\begin{equation}
    \mathrm{SNR}_{w_{\rm pos}}(t) \;=\; 1,
    \label{eq:snr-balanced}
\end{equation}
and the improvement over the unweighted baseline is exactly
\begin{equation}
    \frac{\mathrm{SNR}_{w_{\rm pos}}(t)}{\mathrm{SNR}_1(t)}
    \;=\;
    \frac{n^-_t}{n^+_t}
    \;=\;
    w^+_b(t).
    \label{eq:snr-gain}
\end{equation}
\end{proposition}

\begin{proof}
By Eq.~\ref{eq:wpos}, $w^{ij}_{\rm pos}(t) = w^+_b(t)$ when
$e^{(i,j)}_0=1$ and $w^{ij}_{\rm pos}(t)=1$ when $e^{(i,j)}_0=0$.
Substituting into the numerator of \eqref{eq:snr-def}:
\begin{equation}
    \sum_{\calQ_t} w^{ij}_{\rm pos}(t)\,e^{(i,j)}_0\,Q_{ij}(t)
    \;=\;
    w^+_b(t)\sum_{\calQ_t} e^{(i,j)}_0\,Q_{ij}(t)
    \;=\;
    w^+_b(t)\,n^+_t,
\end{equation}
using the definition of $n^+_t$ (Eq.~\ref{eq:npos_nneg}). Substituting the
definition $w^+_b(t) = n^-_t/\max(n^+_t,1) = n^-_t/n^+_t$ (since
$n^+_t\ge1$ by assumption) gives
\begin{equation}
    w^+_b(t)\,n^+_t \;=\; \frac{n^-_t}{n^+_t}\cdot n^+_t \;=\; n^-_t.
\end{equation}
For the denominator, since $w^{ij}_{\rm pos}(t)=1$ whenever $e^{(i,j)}_0=0$,
\begin{equation}
    \sum_{\calQ_t} w^{ij}_{\rm pos}(t)\bigl(1-e^{(i,j)}_0\bigr)\,Q_{ij}(t)
    \;=\;
    \sum_{\calQ_t}\bigl(1-e^{(i,j)}_0\bigr)\,Q_{ij}(t)
    \;=\;
    n^-_t.
\end{equation}
Hence $\mathrm{SNR}_{w_{\rm pos}}(t) = n^-_t/n^-_t = 1$, proving
\eqref{eq:snr-balanced}. Since $\mathrm{SNR}_1(t)=n^+_t/n^-_t$
(Definition~\ref{def:snr}), the ratio
$\mathrm{SNR}_{w_{\rm pos}}(t)/\mathrm{SNR}_1(t) = n^-_t/n^+_t = w^+_b(t)$
follows immediately, proving \eqref{eq:snr-gain}.
\end{proof}

\begin{remark}[Gain Relative to a Dense-Pool Baseline]
\label{rem:snr-dense}
Let $n^{-,\mathrm{dense}}_t$ denote the negative-class pool mass of a dense
baseline ELBO that pools over all $\binom{N}{2}$ pairs rather than
$\calQ_t$, under an identical per-pair masking rule. Since $\calQ_t\subset
\{1,\dots,\binom N2\}$ and both poolings apply the same masking indicator
$Q_{ij}(t)$ to their respective (larger vs.\ smaller) non-edge populations,
\begin{equation}
    n^{-,\mathrm{dense}}_t \;=\; \Theta\!\left(\binom{N}{2}-m\right),
    \qquad
    n^-_t \;=\; \Theta\bigl(|\calQ_t|-m\bigr) \;=\; \Theta(N\mu_2)
\end{equation}
by Proposition~\ref{prop:coverage}, at any fixed masking fraction. Since
$n^+_t=\Theta(m)$ under either pooling (all true edges are candidates in
both), the dense baseline's unweighted SNR is
\begin{equation}
    \mathrm{SNR}_1^{\mathrm{dense}}(t)
    \;=\; \Theta\!\left(\frac{m}{N^2}\right),
\end{equation}
while our reweighted SNR is exactly $1$
(Proposition~\ref{prop:snr}). The total SNR gain of the reweighted,
restricted-pool objective over an unweighted dense-pool ELBO is therefore
\begin{equation}
    \frac{\mathrm{SNR}_{w_{\rm pos}}(t)}{\mathrm{SNR}_1^{\mathrm{dense}}(t)}
    \;=\;
    \Theta\!\left(\frac{N^2}{m}\right)
    \;=\;
    \Theta\!\left(\frac{1}{\rho}\right),
    \qquad \rho := \frac{m}{\binom N2}
\end{equation}
the inverse graph density — i.e., the sparser the graph, the larger the
SNR gain, which is precisely the regime motivating this work.
\end{remark}

\subsection{Practical analysis}\label{app:practical}
\paragraph{Degree-aware masking behaves as designed.}
Figures~\ref{fig:Cora-degree-aware}-~\ref{fig:Pol-degree-aware} validate the degree-aware retention schedule
$\beta^{ij}=\min(\beta_{\max},(d^{i*}d^{j*}/\bar d^{2})^{\gamma/2})$ against
a uniform baseline ($\beta^{ij}\equiv1$), using closed-form expectations
over each dataset's true graph structure -- no model training or sampling
is involved, so every curve here is an exact property of the schedule
itself, not an empirical estimate.

The two schedules differ in one specific way, worth stating precisely:
under uniform masking, the retention curve $\bar\alpha_t^{\beta^{ij}} =
\bar\alpha_t$ (the probability an edge's true state has not yet been
replaced by the mask token) is identical for every edge, regardless of
what it connects. Degree-aware masking removes this blindness: because
$\beta^{ij}$ depends on $(d^{i*},d^{j*})$, the \emph{rate} of noise
injection is itself informed by an edge's structural role, giving the
denoiser an explicit curriculum built into the schedule rather than one it
must discover on its own.

To make this concrete, we call a node a \emph{hub} if its degree falls in
the top quartile of the dataset's degree distribution, and classify an
edge as hub--hub, hub--leaf, or leaf--leaf by how many endpoints qualify
("leaf" here means any non-hub node -- the bottom three quartiles -- not a
degree-1 node specifically). Panels (a)--(c) show the intended effect is
measurable, not just designed: $\beta^{ij}$ spans nearly an order of
magnitude across edge types (Cora: hub--hub mean $2.14$ vs.\ leaf--leaf
mean $0.68$; Polblogs, the most degree-heterogeneous dataset we test,
$d_{\max}/\bar d\approx12.8$: hub--hub mean $2.68$ vs.\ leaf--leaf mean
$0.64$). This separation carries through to the retention curves in panel
(b), which visibly split by percentile instead of collapsing onto one
line, and resolves into a genuine staged curriculum in panel (c): by the
schedule's midpoint, hub--hub edges have already lost most of their signal
(survival $0.18$ for Polblogs, $0.27$ for Cora), while leaf--leaf edges
remain mostly intact (survival $0.64$--$0.67$) -- against a flat $0.517$
for every edge type under uniform masking, which cannot distinguish them.

    \begin{figure}[t]
        \centering
        \includegraphics[width=\linewidth]{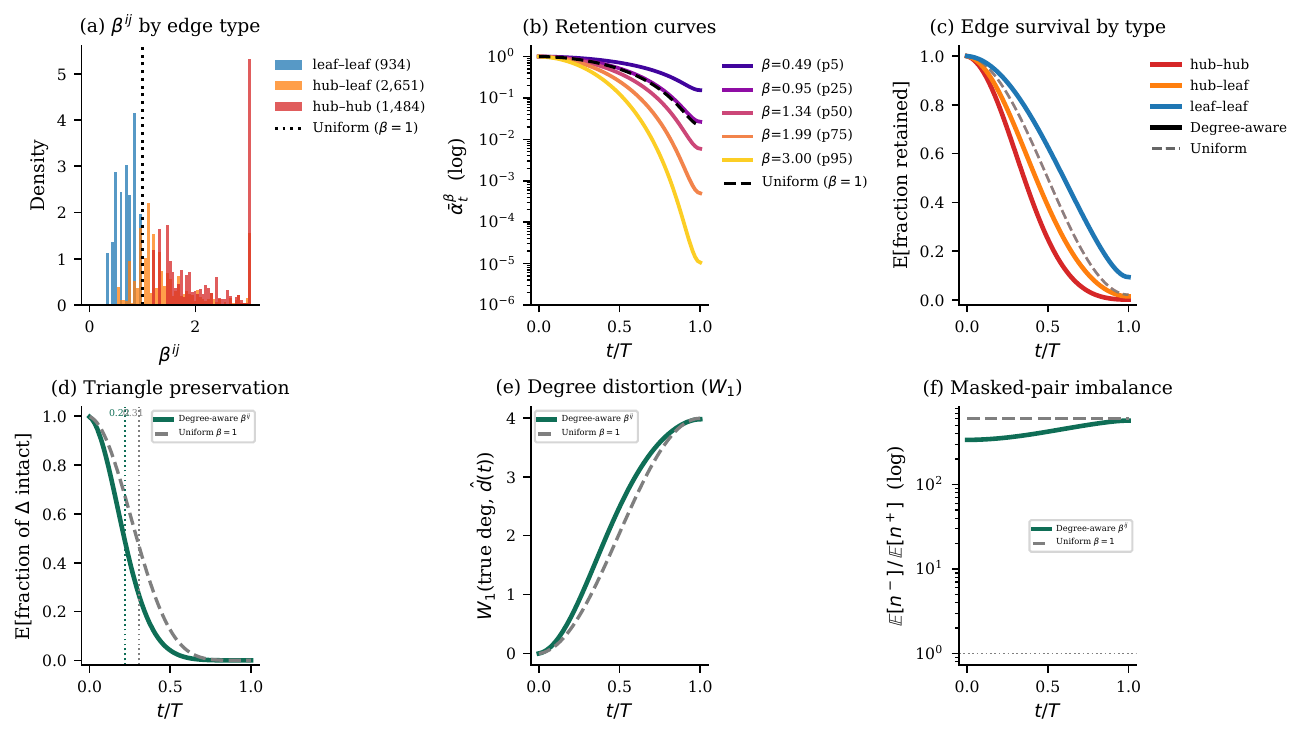}
        \caption{Cora: degree-aware masking analysis}\label{fig:Cora-degree-aware}
    \end{figure}\hfill
    \begin{figure}[t]
        \centering
        \includegraphics[width=\linewidth]{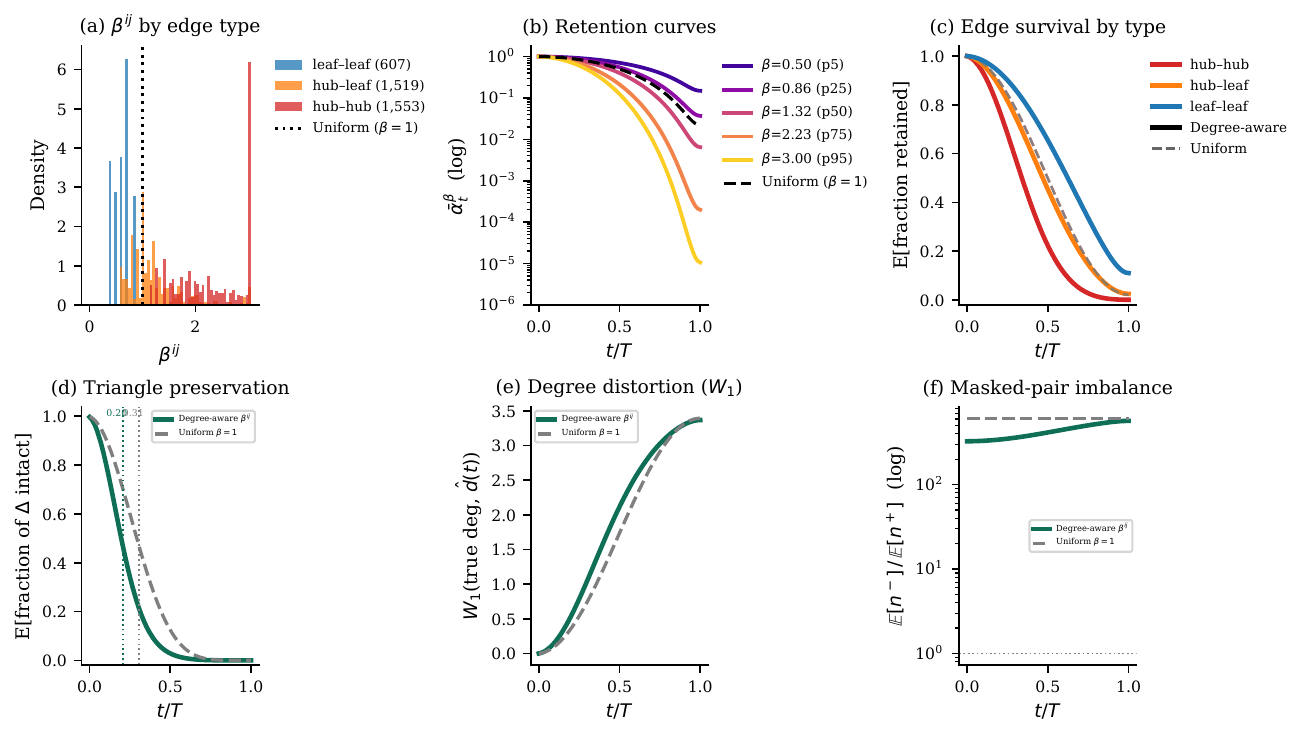}
        \caption{Citeseer: degree-aware masking analysis}\label{fig:Citeseer-degree-aware}
    \end{figure}\hfill
    \begin{figure}[t]
        \centering
        \includegraphics[width=\linewidth]{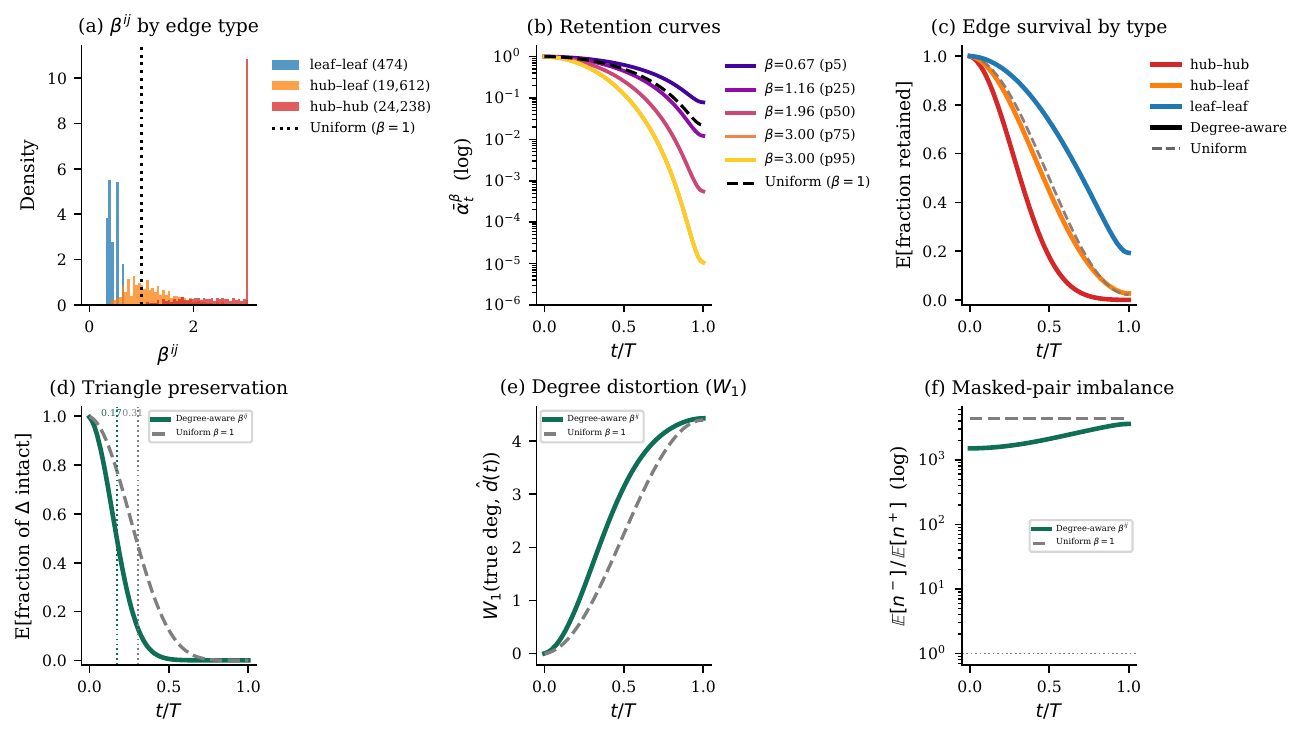}
        \caption{Pubmed: degree-aware masking analysis}\label{fig:Pubmed-degree-aware}
    \end{figure}

    \begin{figure}[t]
        \centering
        \includegraphics[width=\linewidth]{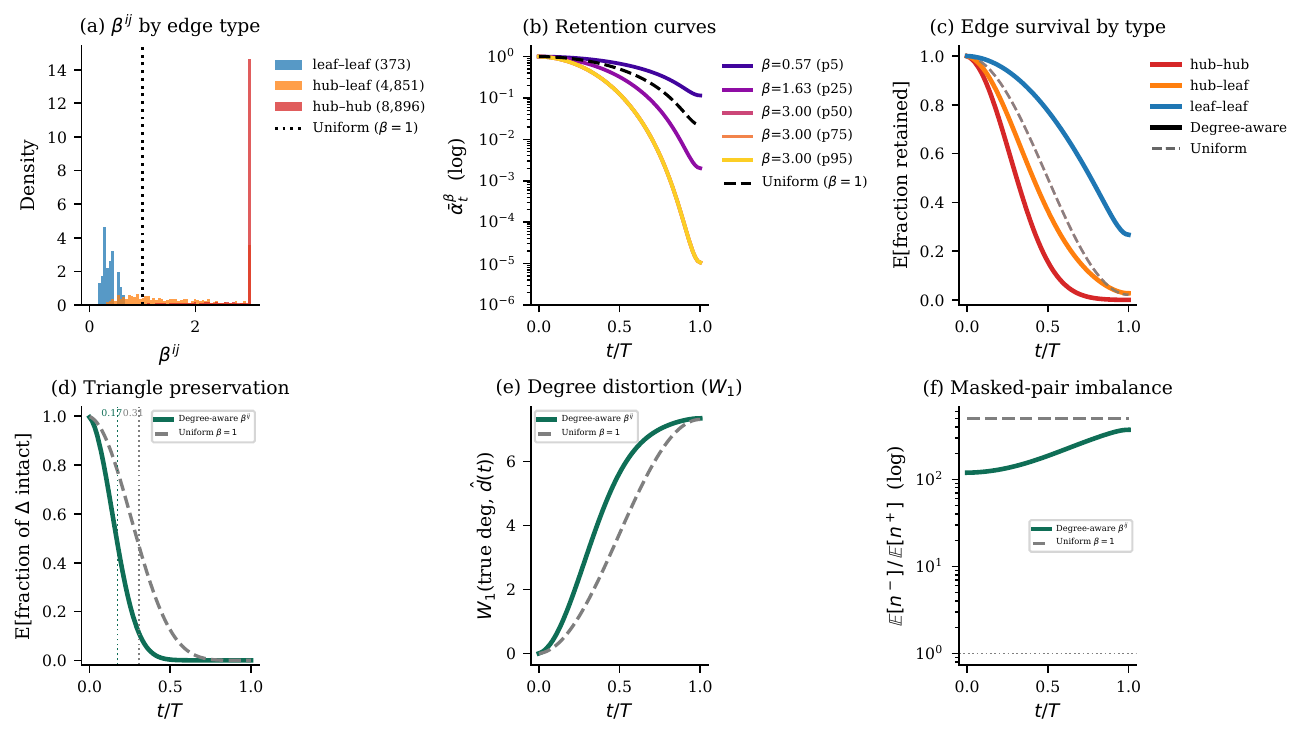}
        \caption{Alpha: degree-aware masking analysis}\label{fig:Alpha-degree-aware}
    \end{figure}\hfill
    \begin{figure}[t]
        \centering
        \includegraphics[width=\linewidth]{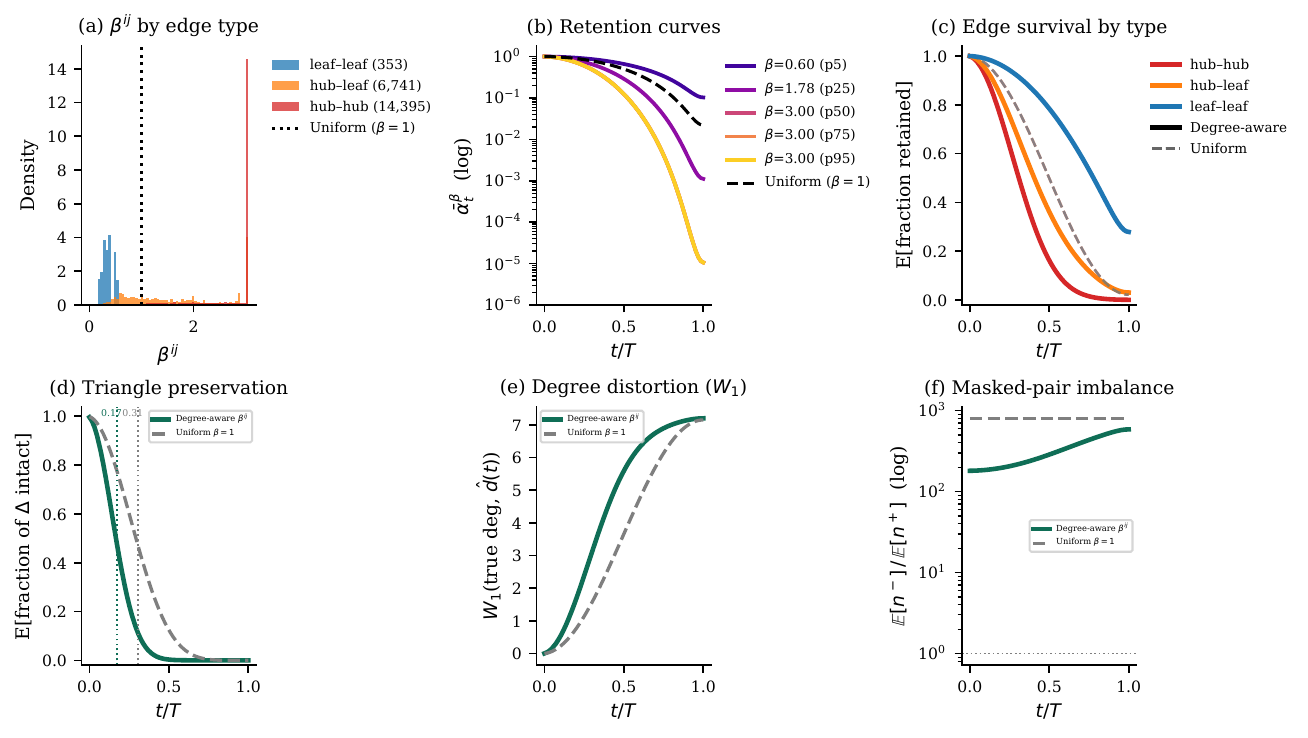}
        \caption{OTC: degree-aware masking analysis}\label{fig:OTC-degree-aware}
    \end{figure}\hfill
    \begin{figure}[t]
        \centering
        \includegraphics[width=\linewidth]{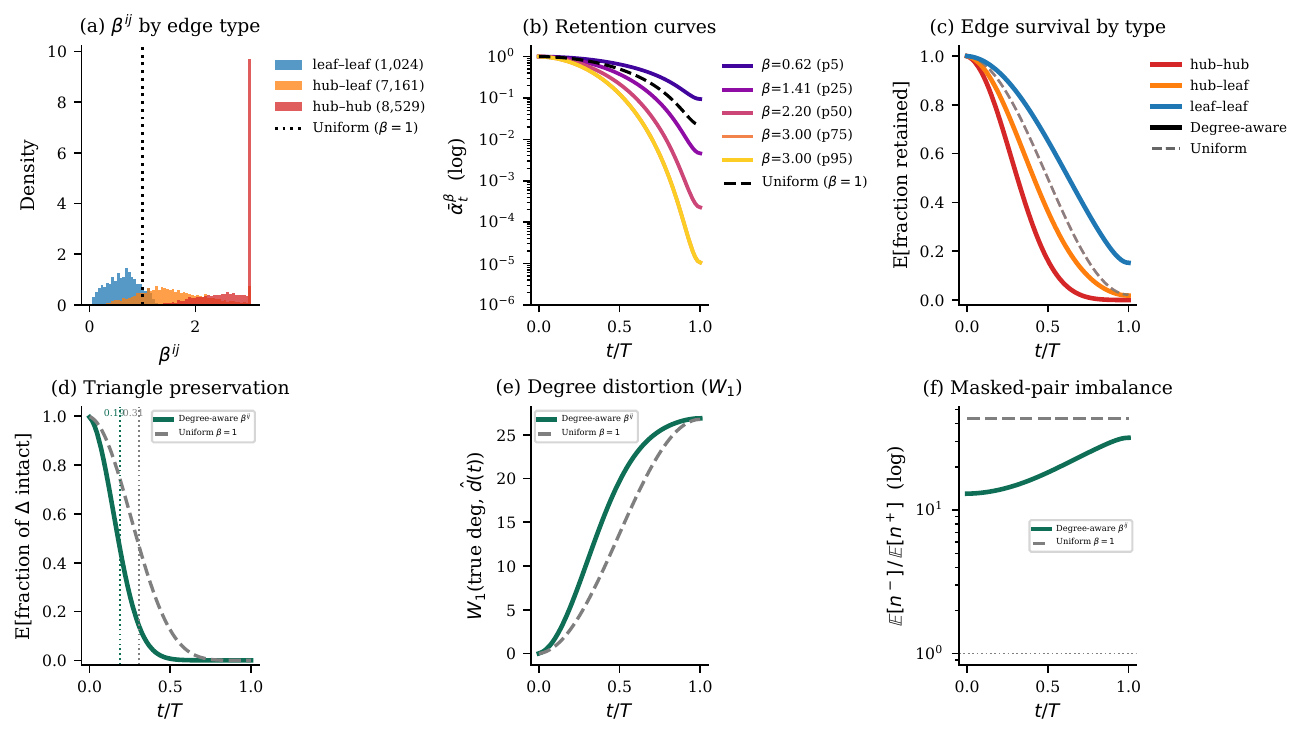}
        \caption{Polblogs: degree-aware masking analysis}\label{fig:Pol-degree-aware}
    \end{figure}

\paragraph{The trade-off is real.}
Front-loading hub structure is not free. Panel (d) measures the expected
fraction of each graph's true triangles whose three constituent edges are
\emph{all} still unmasked at time $t$ -- how quickly triangle-completing
structure is consumed by the forward process, and the component of graph
structure most relevant to clustering coefficient, $CC=3|\triangle|/
|\text{wedges}|$ (panel (d) tracks the numerator's survival only, since it
does not separately model wedge survival).

The \emph{triangle-intact half-life} -- the value of $t/T$ at which this
fraction first drops below one half -- falls from $0.31$ under uniform
masking (identical across datasets by construction) to $0.19$--$0.22$
under degree-aware masking, with the intact fraction at the midpoint
$3$--$16\times$ lower, worst for Polblogs. Since hub-adjacent edges
dominate real triangles, this is the same front-loading from panel (c)
acting on triangles specifically.
This early half-life offers a plausible explanation for a separate finding
in our step-reduction ablation (\S\ref{appx:step_reduction}): sampling
with only 50\% of the trained step count produced almost no quality loss.
Because $\bar\alpha_t^{\beta^{ij}}$ is a fixed function of $t/T$
independent of how many discrete steps are used to traverse it,
triangle-relevant structure settles into its final form at a consistent
point along the schedule regardless of step count. If that point falls
within the region step-reduction leaves densely sampled, this could
explain why coarser step counts preserve quality; we flag this as a
hypothesis connecting two separate observations, not a claim we test
directly here.

Panel (e) shows a related cost in the degree distribution itself. At each
$t$, we compute each node's \emph{expected partially-revealed degree},
$\tilde d^i(t)=\sum_{j\in\mathcal N(i)}\bar\alpha_t^{\beta^{ij}}$ under
degree-aware masking (equivalently $d^{i*}\bar\alpha_t$ under uniform) --
the number of true edges at $i$ still unmasked, in expectation, at that
point in the schedule -- and compare the resulting distribution across all
nodes to the true degree distribution via Wasserstein-1 distance. We use a
distributional distance rather than a per-node error because what matters
is whether the \emph{shape} of the degree distribution is preserved at
that checkpoint, not whether any single node's partial degree matches its
final value exactly.

By this measure, degree-aware masking is consistently \emph{more}
distorted than uniform at the midpoint ($1.23\times$ for Cora, $1.46\times$
for Polblogs): because hub and leaf edges resolve on different timescales,
a node's partially-revealed degree becomes a skewed mixture of
early-resolved leaf edges and late-resolved hub edges, rather than the
simple rescaling of true degree that uniform masking's single shared rate
produces by default.

\paragraph{Reweighting is what makes this trade-off pay off on real data.}
Panel (f) is where degree-aware masking earns its keep in practice. At
each $t$, a true edge not yet revealed contributes to $n^+_t$ (masked
edges); a true non-edge not yet revealed contributes to $n^-_t$ (masked
non-edges); we report their expected ratio $\mathbb E[n^-_t]/
\mathbb E[n^+_t]$ -- the same quantity the class-balance weight
$w^+_b(t)$ (Eq.~\ref{eq:wpos}) is designed to correct. Note that $n^+_t,
n^-_t$ here are computed as closed-form expectations over the \emph{full}
graph, not over the restricted training-time candidate pool
$\calQ_t=\calQ\cup S_t$; the two are
related but not identical, since this analysis is model-free and makes no
reference to the training-time restriction to $\calQ_t$.

Graphs at the scale we target are extremely sparse (Cora
$\rho\approx0.0016$; Pubmed $\rho\approx0.0002$), so even a moderate
masking rate leaves an overwhelming excess of masked non-edges relative to
masked edges. Degree-aware masking roughly halves this imbalance relative
to uniform at every dataset we test: $0.69\times$ for Cora, $0.67\times$
for Citeseer, and $0.41\times$ for Polblogs -- the reduction growing with
degree heterogeneity, the same way the panel (e) cost does. The two
mechanisms are therefore not independent design choices: degree-aware
masking does not eliminate the sparsity-driven imbalance inherent to real
graph data, but it substantially lightens the correction $w^+_b(t)$ has to
make, letting the two work together rather than asking reweighting alone
to compensate for a uniform schedule's worst-case imbalance.

\paragraph{Why this matters.} Giving
the schedule a sense of which edges matter structurally lets it teach the
model something a uniform schedule simply cannot: hub connectivity first,
periphery later. That choice costs something -- triangles and degree
shape both take a hit early in the trajectory, and that cost shows up
later as a real, measurable gap in clustering coefficient, even in our
best runs. But it also buys something that matters more for training on
graphs this sparse: it substantially reduces the pool-level class
imbalance that the reweighting term would otherwise have to correct on
its own, doing part of that work for free before reweighting is even
applied. Neither mechanism was designed to fully cover for the other, but
together they leave the model in a measurably better-balanced training
pool than either alone would produce -- and that pattern held everywhere
we looked, scaling with exactly the property, degree heterogeneity, that a
degree-\emph{aware} schedule ought to respond to.

\section{Extended Results}\label{sec:a_extend_result}
We provide a comprehensive evaluation across a diverse range of datasets,
from small multi-graph collections to large single graphs, spanning
citation and trust networks, and report sampling time alongside generation
quality.

Figures~\ref{fig:ap_node_degree_vilon} and~\ref{fig:ap_node_degreeccd} show
the degree distribution of generated nodes for the large single-graph
datasets. Our model closely tracks the training distribution across most
of the degree range -- the violin plots in
Figure~\ref{fig:ap_node_degree_vilon} show node density per degree value --
but, as Figure~\ref{fig:ap_node_degreeccd} shows, it underrepresents the
long-tail (high-degree) region on our larger datasets. Despite this,
Table~\ref{tab:data_status} shows our model achieves the best clustering
coefficient recovery among all methods compared, indicating strong local
structural fidelity even where the global degree tail is harder to match.

\subsection{Results on Small Graphs}\label{appx:small_res}
Although our method targets large-graph generation, Table~\ref{tab:small_metrics}
shows it achieves performance comparable to methods specifically designed
for and limited to small-scale graphs, most of which do not scale to the
large-graph setting that is our primary focus.

On the ZINC molecular dataset, our model is competitive with dedicated
molecule generation methods. As shown in Table~\ref{tab:zinc_metrics}, it
achieves the best performance on all reported metrics except orbits and
validity. We note that validity (\S~\ref{sec:validity}) here refers to \emph{structural} validity
-- well-formedness of the generated graph -- rather than chemical validity
(e.g., valence-consistent bonding); our model's strong performance on other
metrics should be read with this scope in mind.

\begin{figure*}[t!]
    \centering

    \begin{subfigure}[t]{0.3\textwidth}
        \centering
        \includegraphics[width=\textwidth]{cora_degree_distribution_violin.pdf}
        
    \end{subfigure}
    \hfill
    \begin{subfigure}[t]{0.3\textwidth}
        \centering
        \includegraphics[width=\textwidth]{citeseer_degree_distribution_violin.pdf}
        
    \end{subfigure}
    \hfill
    \begin{subfigure}[t]{0.3\textwidth}
        \centering
        \includegraphics[width=\textwidth]{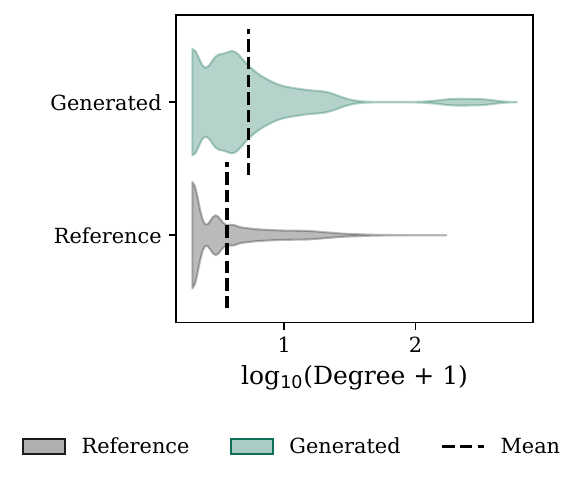}
        
    \end{subfigure}
    \hfill
    \begin{subfigure}[t]{0.3\textwidth}
        \centering
        \includegraphics[width=\textwidth]{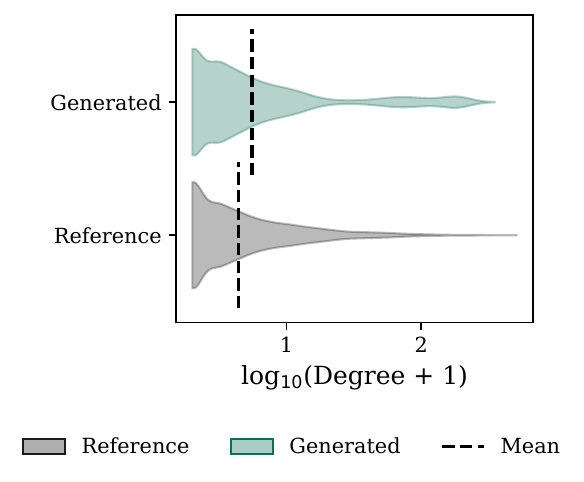}
        
    \end{subfigure}
    \hfill
    \begin{subfigure}[t]{0.3\textwidth}
        \centering
        \includegraphics[width=\textwidth]{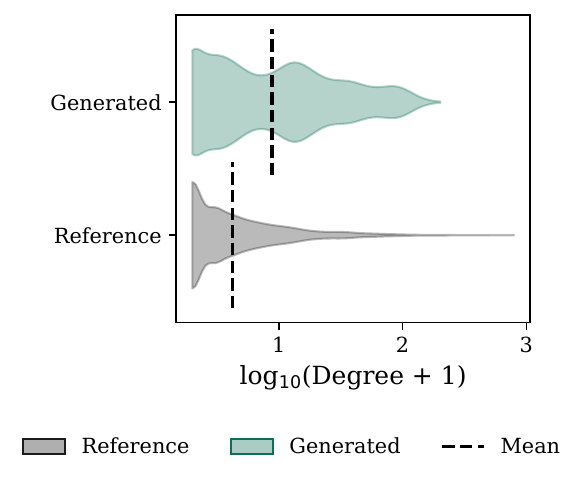}
        
    \end{subfigure}
    \hfill
    \begin{subfigure}[t]{0.3\textwidth}
        \centering
        \includegraphics[width=\textwidth]{pol_degree_distribution_violin.pdf}
        \end{subfigure}
    \caption{\small Degree distribution comparison across six datasets; plot width indicates node density.}
    \label{fig:ap_node_degree_vilon}
\end{figure*}

\begin{figure*}[t!]
    \centering

    \begin{subfigure}[t]{0.3\textwidth}
        \centering
        \includegraphics[width=\textwidth]{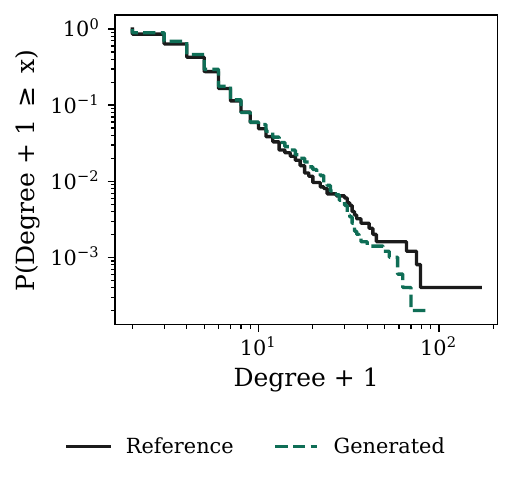}
        
    \end{subfigure}
    \hfill
    \begin{subfigure}[t]{0.3\textwidth}
        \centering
        \includegraphics[width=\textwidth]{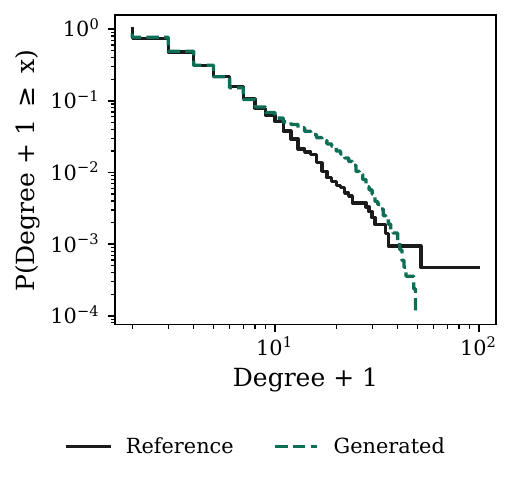}
        
    \end{subfigure}
    \hfill
    \begin{subfigure}[t]{0.3\textwidth}
        \centering
        \includegraphics[width=\textwidth]{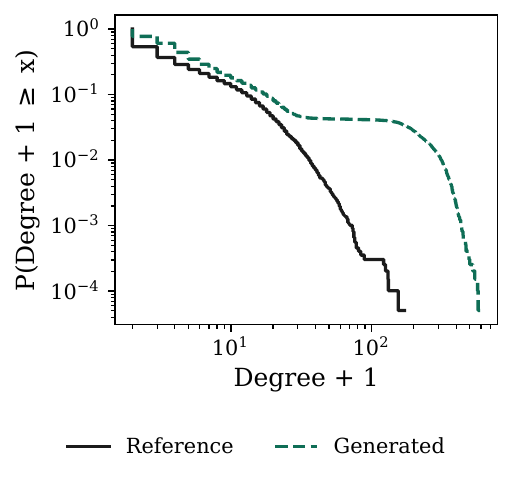}
        
    \end{subfigure}
    \hfill
    \begin{subfigure}[t]{0.3\textwidth}
        \centering
        \includegraphics[width=\textwidth]{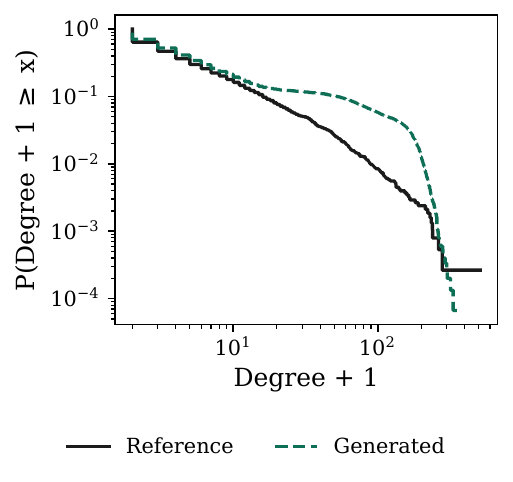}
        
    \end{subfigure}
    \hfill
    \begin{subfigure}[t]{0.3\textwidth}
        \centering
        \includegraphics[width=\textwidth]{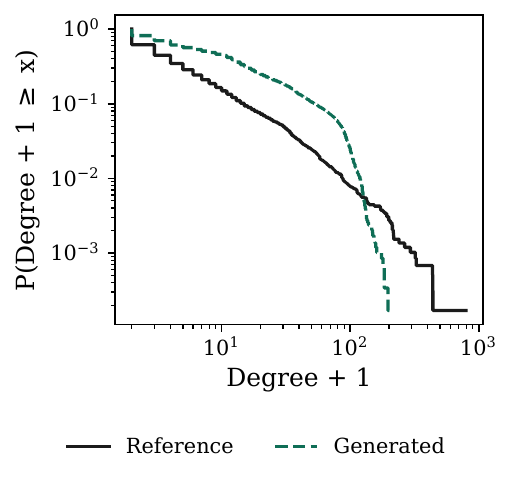}
        
    \end{subfigure}
    \hfill
    \begin{subfigure}[t]{0.3\textwidth}
        \centering
        \includegraphics[width=\textwidth]{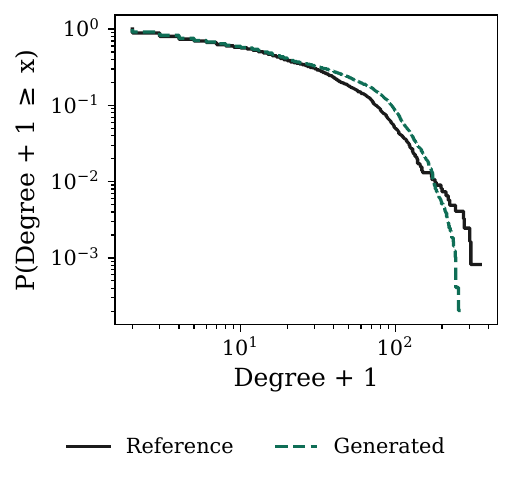}
        \end{subfigure}
    \caption{\small Degree distribution comparison across six datasets; plot width indicates node density.}
    \label{fig:ap_node_degreeccd}
\end{figure*}

\subsection{Empirical sampling time}
As discussed in Appendix~\ref{appx:Analysis}, our model achieves
sub-quadratic \emph{training} complexity, but sampling remains $O(N^2)$
per graph, the same asymptotic order as the diffusion baselines we
compare against. Here we examine whether this translates into a practical
bottleneck. We measure sampling (generation) time across graphs of varying
size and compare it to other methods. As shown in
Figure~\ref{fig:time_sampling}, our model generates graphs with more than
10k nodes in under 110 seconds, while EDGE and GraphMaker require
approximately 200 seconds at the same graph size. ARROW-Diff achieves
comparable speed to our model at this scale, but this comes at a different
cost: it performs multiple rounds of random-walk sampling per graph, and,
as shown in Table~\ref{tab:data_status}, produces substantially higher
edge overlap (EO) than any other method we evaluate -- a strong indicator
of edge memorization rather than genuine structural generalization. Our
model achieves comparable sampling time without this trade-off, generating
graphs with low edge overlap alongside strong recovery of other structural
statistics.

\begin{figure}[H]
    \centering
    \includegraphics[width=0.7\textwidth]{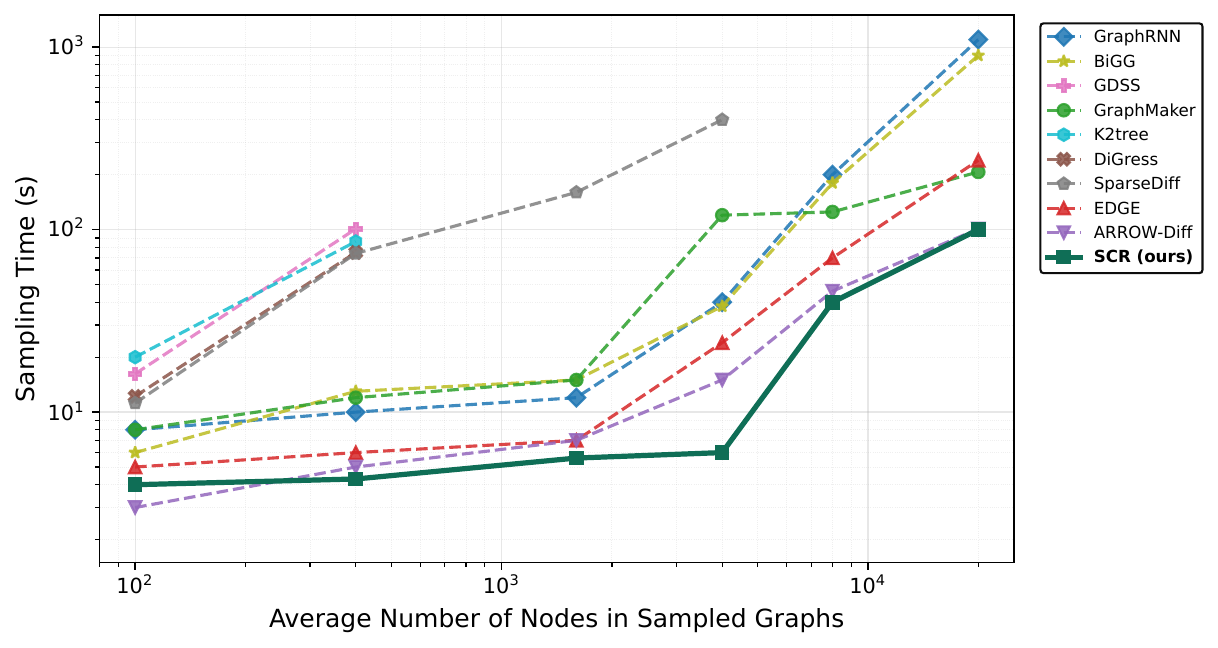}
   \caption{\small Sampling time across various graph sizes, averaged over 16 runs per model under identical conditions. }\label{fig:time_sampling}
    
\end{figure}

\subsection{Efficiency of sampling: reduced diffusion steps}
\label{appx:step_reduction}

\paragraph{Motivation.} While our method achieves sub-quadratic training
cost, sampling remains $O(N^2)$ per graph, the same asymptotic order as
the dense baselines we compare against. For a method whose central goal is
scaling to large graphs, this asymmetry is a real practical concern: the
training-time savings are of limited use if generating a single large
graph still takes disproportionately long. Reducing the number of reverse
diffusion steps is the most direct lever available to lower this per-graph
cost without further architectural changes, so it is important to
establish how much quality is actually lost when doing so -- and whether
the resulting time savings are worth the trade-off in practice.

\begin{figure}
    \centering
    \includegraphics[width=\linewidth]{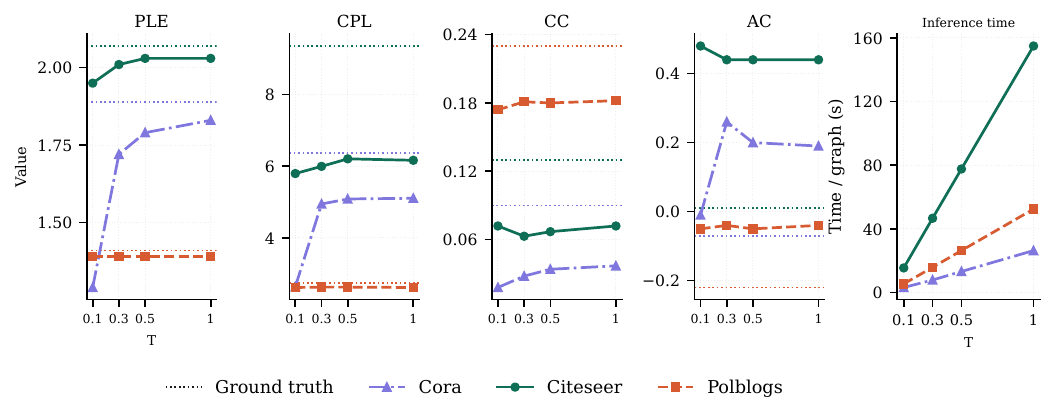}
    \caption{ Structural statistics of generated graphs vs. inference-time diffusion steps.}
    \label{fig:inference_time}
\end{figure}
Using a single model trained with $T_{\rm train}$ steps, we evaluate
generation quality when sampling with $T\in\{10\%,30\%,50\%,100\%\}$ of
$T_{\rm train}$ reverse steps, skipping directly between the sampled $t$
values rather than truncating a prefix.

\paragraph{Why step-skipping is exact.} This introduces no approximation in
the reverse update rule itself: $r_{t\to t_{\rm prev}}^{ij}$
(Eq.~\ref{eq:unmask}) is derived for arbitrary $t_{\rm prev}\le t$, not only
single-step transitions. What \emph{can} change at coarser $T$ is the
distribution of intermediate states $G_t$ the encoder is queried on --
fewer, more widely spaced checkpoints than the trajectory seen during
training. This experiment therefore tests the model's robustness to that
intermediate-state shift, not an error in the sampling procedure.

\paragraph{Results.} We illustrate the pattern using Cora, Citeseer, and
Polblogs as a representative example (Figure~\ref{fig:inference_time} and
Table~\ref{tab:step_reduction_grouped}); the same qualitative behavior is
seen for the other datasets and settings we tested. Sampling time
decreases roughly in proportion to $T$, while generation quality --
measured via power-law exponent (PLE), characteristic path length (CPL),
and clustering coefficient (CC) -- is largely preserved down to
$T{\approx}50\%$. Degradation becomes pronounced only at $T{=}10\%$, and
is more severe for datasets trained with a smaller $T_{\rm train}$ (e.g.\
Cora, $T_{\rm train}{=}64$) than for those trained with a larger one
(Citeseer and Polblogs, $T_{\rm train}{=}512$), even though both are being
reduced to the same \emph{fraction}.

\paragraph{Outcome: reduced inference steps and degree-aware masking.}
Across all six ablation datasets, truncating the number of reverse
sampling steps to a fraction of $T$ preserves generation quality
substantially better under the degree-adaptive schedule $\beta^{ij}$ than
under uniform masking ($\beta^{ij}=1$). Since
$\beta^{ij}=\min(\beta_{\max},(d^{i*}d^{j*}/\bar d^2)^{\gamma/2})$
increases with the degree product, $\bar\alpha_t^{\beta^{ij}}$ decays
\emph{faster} for hub--hub pairs, so they cross any fixed retention
threshold at \emph{smaller} $t$ in the forward process and are
consequently the pairs masked \emph{longest} -- and therefore resolved
\emph{last} -- in the reverse chain, while leaf--leaf pairs
($\beta^{ij}\ll1$) are masked latest in the forward direction and
resolved first in reverse. This extends the forward-process analysis of
\S~\ref{app:practical} to the reverse sampling
trajectory, and the two are consistent: the same pairs shown there to
survive longest in the forward direction are shown here to be resolved
last in reverse. We confirm this directly with a Monte Carlo simulation
of the exact per-edge unmask hazard used at sampling time,
$r^{ij}_{t\to t_{\mathrm{prev}}} = (\bar\alpha_{t_{\mathrm{prev}}}^{\beta^{ij}} -
\bar\alpha_t^{\beta^{ij}})/(1-\bar\alpha_t^{\beta^{ij}})$, run over the true
$T{-}1,\dots,0$ reverse step sequence with 50{,}000 independent trials
per edge class. Leaf--leaf edges (5th percentile of each dataset's true
per-edge $\beta^{ij}$ distribution) resolve at $44.7$--$49.0\%$ of the
way through the reverse trajectory on average; hub--hub edges (95th
percentile) resolve at $50.6$--$57.7\%$ -- later in every one of the six
datasets (Cora, CiteSeer, PubMed, Polblogs, Bitcoin Alpha, Bitcoin OTC),
with no exceptions.

What reduced-step truncation actually preserves, then, is not privileged
attention to hubs, but the fact that degree-aware masking is the only one
of the schedules considered that assigns different edge classes to
\emph{different, degree-correlated} points in the trajectory at all. A
$K$-step, uniformly-spaced skip schedule over $[0,T-1]$ always includes
points near both $t\approx T$ and $t\approx 0$ regardless of $K$; under
degree-aware masking this means even a heavily truncated schedule still
separately captures an early, mostly-leaf-resolved scaffold and a late,
hub-correcting refinement. Under uniform masking every edge shares the
same marginal $\bar\alpha_t$, so there is no such temporal separation for
truncation to exploit -- coarsening the step count degrades every edge
class identically rather than trading resolution on a subset that has
already received comparatively more of the trajectory to condition on. We
present this as the most parsimonious explanation consistent with the
confirmed reveal-order data, not as an independently verified mechanism;
isolating it from competing explanations would require comparing
generation quality broken down by edge degree class at fixed truncated
$K$, which we leave to future work.

The strength of this effect is also dataset-dependent in a way that should
be stated explicitly rather than averaged away. The six datasets split
into two regimes governed by how much headroom $\beta_{\max}$ leaves above
the typical degree product: Cora, CiteSeer, and PubMed (all
$\beta_{\max}=1.5$, $41$--$63\%$ of edges at the cap) show a consistent
$10.2$--$13.0$ percentage-point gap between leaf and hub reveal time;
Polblogs, Bitcoin Alpha, and Bitcoin OTC show only $3.9$--$4.7$ points.
For the two Bitcoin trust networks this is not merely "more saturation"
but a degenerate special case: $\beta_{\max}=1.0$ is numerically identical
to the uniform baseline $\beta^{ij}=1$, so the $86.7$--$88.1\%$ of edges
sitting at the cap are, by construction, scheduled \emph{exactly as if
degree-awareness were switched off} -- only the uncapped low-degree tail
($\beta^{ij}_{p5}\approx0.57$--$0.60$) is scheduled any differently from
uniform at all, pulled earlier rather than the hub tier being pulled
later. Polblogs uses the same $\beta_{\max}=1.5$ as the citation networks
yet still falls in the weak-separation group, because its exceptionally
high mean degree ($\bar d=27.36$, an order of magnitude above the other
datasets) pushes even its \emph{median} edge to the cap, not only its
tail.

\subsection{Effect of the prior scaffold.}\label{appx:Scafold}

We compare three terminal-distribution priors for the visible pairs at
$t=T$, holding the degree-aware masking exponent $\beta^{ij}$
(Eq.~\ref{eq:beta}) identical across all three, so that only the
\emph{value} distribution of revealed pairs differs:
\begin{itemize}
\item \textbf{Erd\H{o}s--R\'enyi (ER):} $\rho^{ij} = p := m/\binom N2$ for
  all pairs --- equivalently, every pair is an i.i.d.\ $\mathrm{Bernoulli}(p)$
  draw with a single shared parameter, uniform over all $\binom N2$ pairs
  regardless of node degree. Since $p = \bar d/(N-1)$, where
  $\bar d = 2m/N$ is the mean degree, this is the degree-blind special case
  of Eq.~\ref{eq:LT} obtained by replacing $d^{i*},d^{j*}$ with $\bar d$
  for every node.
\item \textbf{Stochastic Block Model (SBM):} $\rho^{ij} = B_{z_i,z_j}$, with
  block assignments $z_i$ taken from each dataset's ground-truth class
  labels and $K$ set to the dataset's native number of classes (Cora:
  $K=7$; Polblogs and Alpha: $K=2$); $B_{kl}$ is fit
  as the empirical block-pair edge density on the training graph.
\item \textbf{Chung--Lu (ours, default):} $\rho^{ij}=d^{i*}d^{j*}/2m$
  (Eq.~\ref{eq:LT}), as used throughout the main results.
\end{itemize}
Table~\ref{tab:ablation_consolidated} holds $\beta^{ij}$ fixed across all three arms, so the reveal \emph{order} is identical and only the reveal \emph{value} $\rho^{ij}$ differs --- isolating what the model is told about the hub pairs that degree-adaptive masking always exposes first. Structurally, only the Chung--Lu prior scales with per-node degree, so it is the only choice whose functional form is consistent with that ordering: it assigns high edge probability to exactly the pairs $\beta^{ij}$ has already prioritized, while ER's flat $p$ and SBM's block average $B_{z_i,z_j}$ anchor those same pairs to an uninformative statistic. This structural advantage translates into an empirical one in most, though not all, cases: on Alpha, which has no natural block structure, SCR sweeps all four metrics while ER and SBM overestimate clustering by $7$--$8\times$. On Polblogs, whose near-binary ideological structure happens to align with the SBM's own inductive bias, SBM edges out SCR on three metrics --- the one condition structurally favorable to it. What we want to emphasize here is that graph-generation priors can be anchored to a well-known reference distribution --- as with ER, SBM, or Chung--Lu --- but depending on the behavior of the noise schedule governing the trajectory between the original and noised graph, some priors are better suited than others and lead to higher-quality generation. In the majority of cases, Chung--Lu is best suited to our degree-aware forward process, since both are governed by the same underlying quantity: node degree.
\subsection{Effect of the regularizers.}\label{appx:regularizers}
Table~\ref{tab:ablation_consolidated} ablates the two structural regularizers, $\mathcal{L}_{deg}$ and $\mathcal{L}_{tri}$, against default SCR on Cora, Polblogs, and Bitcoin-Alpha. The full model is best or tied-best in $10$ of $12$ dataset-metric cells, with two explainable exceptions discussed below. Removing $\mathcal{L}_{deg}$ consistently and substantially degrades clustering-coefficient recovery on all three datasets ($|\Delta_{CC}|$ increases $2$--$4\times$), which is worth stating precisely because $CC = 2\cdot\text{triangles}/(d(d-1))$ is degree-normalized: an incorrect degree distribution distorts $CC$ even when triangle-matching pressure from $\mathcal{L}_{tri}$ and the HT-reweighted candidate-set gradient remain active, making $\mathcal{L}_{deg}$ a precondition for correct clustering recovery rather than a term that merely competes with it. Conversely, removing $\mathcal{L}_{tri}$ consistently degrades assortativity on all three datasets ($|\Delta_{AC}|$ nearly doubles on Cora and Polblogs) — evidence that $\mathcal{L}_{tri}$'s practical role is shaping degree-correlation and mixing structure rather than literal triangle-count matching. The one dataset-specific exception is Bitcoin-Alpha, where $\mathcal{L}_{deg}$ improves every other metric but degrades AC ($|\Delta|$ $0.12\to0.27$); given this dataset's extreme degree heterogeneity (max degree $511$ on $N{=}3{,}775$, versus $\bar d\approx4$--$27$ elsewhere), we attribute this to tension between the degree-matching gradient and the mixing structure induced by such an extreme tail, and report it as an observed effect rather than a fully diagnosed mechanism. The power-law exponent is essentially unaffected by either ablation on Cora and Polblogs ($|\Delta_{PLE}|\le0.02$ in every row), consistent with the degree-adaptive masking schedule $\beta^{ij}$ — held fixed across all three arms — being the primary driver of overall degree-distribution shape, with $\mathcal{L}_{deg}$ and $\mathcal{L}_{tri}$ acting as finer, dataset-dependent corrections on top of it.

\begin{table*}[t]
\centering
\small
\caption{\small
Comparison of metrics for small graphs. MMD values are reported as \(\times10^{-2}\) (lower is better); FID\(\downarrow\) is based on a neural feature space; \textit{Uniq}\(\uparrow\) indicates the percentage of non-isomorphic generated graphs; and \textit{Nov}\(\uparrow\) measures the percentage of generated graphs not present in the training set (both higher are better).
}
\label{tab:small_metrics}
\begin{adjustbox}{max width=\textwidth}
\begin{tabular}{l
|ccc|ccc
|ccc|ccc}
\toprule
& \multicolumn{6}{c|}{\textbf{Community}} 
& \multicolumn{6}{c}{\textbf{Ego}} \\
\cmidrule(lr){2-7} \cmidrule(lr){8-13}

& \multicolumn{3}{c|}{\textbf{MMD} (\(\times10^{-2}\))\(\downarrow\)} 
& FID\(\downarrow\) & Uniq\(\uparrow\) & Nov\(\uparrow\)
& \multicolumn{3}{c|}{\textbf{MMD} (\(\times10^{-2}\))\(\downarrow\)} 
& FID\(\downarrow\) & Uniq\(\uparrow\) & Nov\(\uparrow\) \\
\textbf{Method} 
& Deg. & Clus. & Orb. 
&  & (\%) & (\%) 
& Deg. & Clus. & Orb. 
&  & (\%) & (\%) \\
\midrule

GraphRNN~\citeyear{Graphrnn}
& 15.23 {\scriptsize\(\pm\)0.41} & 9.55 {\scriptsize\(\pm\)0.23} & 2.97 {\scriptsize\(\pm\)0.18} & 8.38 {\scriptsize\(\pm\)0.22} & 72.4 {\scriptsize\(\pm\)1.8} & 68.9 {\scriptsize\(\pm\)2.1}
& 7.88 {\scriptsize\(\pm\)0.37} & 101.45 {\scriptsize\(\pm\)2.4} & 14.87 {\scriptsize\(\pm\)0.53} & 87.76 {\scriptsize\(\pm\)0.74} & 65.3 {\scriptsize\(\pm\)2.1} & 62.4 {\scriptsize\(\pm\)2.3} \\
BiGG~\citeyear{BiGG}& 14.91 {\scriptsize\(\pm\)0.17}& 8.53 {\scriptsize\(\pm\)0.29}& 9.03 {\scriptsize\(\pm\)0.13} & 7.35 {\scriptsize\(\pm\)0.32} &  71.4 {\scriptsize\(\pm\)0.2} &68.9 {\scriptsize\(\pm\)0.1}& 9.21 {\scriptsize\(\pm\)0.2} & 27.52 {\scriptsize\(\pm\)0.3} &13.90 {\scriptsize\(\pm\)0.2}& 34.23 {\scriptsize\(\pm\)0.1} &  67.8 {\scriptsize\(\pm\)0.1}& 65.5 {\scriptsize\(\pm\)0.3}\\
GDSS~\citeyear{jo2022score} & 5.33 {\scriptsize\(\pm\)0.1}& 10.11 {\scriptsize\(\pm\)0.09}& 4.73 {\scriptsize\(\pm\)0.07}&11.17 {\scriptsize\(\pm\)0.1} &94.5 {\scriptsize\(\pm\)0.01}& 88.6 {\scriptsize\(\pm\)0.03}& 6.82 {\scriptsize\(\pm\)0.03}& 5.33 {\scriptsize\(\pm\)0.01}& 12.71 {\scriptsize\(\pm\)0.02}& 16.83 {\scriptsize\(\pm\)0.1}& 89.3 {\scriptsize\(\pm\)0.01}& 88.9 {\scriptsize\(\pm\)0.1}\\

DiGress~\citeyear{Digress}
& 3.07 {\scriptsize\(\pm\)0.11} & \textbf{1.69} {\scriptsize\(\pm\)0.06} & \underline{2.89} {\scriptsize\(\pm\)0.08} & 4.07 {\scriptsize\(\pm\)0.09} & \underline{98.6} {\scriptsize\(\pm\)0.3} & 98.1 {\scriptsize\(\pm\)0.5}
& {6.09} {\scriptsize\(\pm\)0.15} & \textbf{0.97} {\scriptsize\(\pm\)0.05} & 11.09 {\scriptsize\(\pm\)0.22} & 16.76 {\scriptsize\(\pm\)0.31} & 99.1 {\scriptsize\(\pm\)0.1} & 97.7 {\scriptsize\(\pm\)0.2} \\

EDGE~\citeyear{EDGE}
& \textbf{1.63} {\scriptsize\(\pm\)0.05} & {7.01} {\scriptsize\(\pm\)0.27} & \textbf{2.10} {\scriptsize\(\pm\)0.04} & \textbf{2.24} {\scriptsize\(\pm\)0.06} & \textbf{99.9} {\scriptsize\(\pm\)0.0} & \textbf{98.8} {\scriptsize\(\pm\)0.1}
& \underline{5.59} {\scriptsize\(\pm\)0.09} & 16.61 {\scriptsize\(\pm\)0.48} & \textbf{6.21} {\scriptsize\(\pm\)0.05} & \underline{14.59} {\scriptsize\(\pm\)0.12} & \underline{99.5} {\scriptsize\(\pm\)0.1} & \underline{98.6} {\scriptsize\(\pm\)0.1} \\
SparsDiff~\citeyear{SparsDiff}
& 4.05 {\scriptsize\(\pm\)0.13} & \underline{1.78} {\scriptsize\(\pm\)0.03} & {2.99} {\scriptsize\(\pm\)0.03} & \underline{3.97} {\scriptsize\(\pm\)0.09} & {97.6} {\scriptsize\(\pm\)0.4} & 97.1 {\scriptsize\(\pm\)0.3}
& {7.02} {\scriptsize\(\pm\)0.13} & {5.3} {\scriptsize\(\pm\)0.03} & 11.10{\scriptsize\(\pm\)0.22} & 15.76 {\scriptsize\(\pm\)0.11} & 98.1 {\scriptsize\(\pm\)0.1} & 97.8 {\scriptsize\(\pm\)0.1} \\
K2tree~\citeyear{k2tree} & \underline{3.02} {\scriptsize\(\pm\)0.1} & {2.30} {\scriptsize\(\pm\)0.13} & 2.93 {\scriptsize\(\pm\)0.31} & {5.17}{\scriptsize\(\pm\)0.02} & 95.2 {\scriptsize\(\pm\)0.01} & 48.2 {\scriptsize\(\pm\)0.02} & 7.13 {\scriptsize\(\pm\)0.07} &  \underline{4.31} {\scriptsize\(\pm\)0.01} &  11.70{\scriptsize\(\pm\)0.05}  & 16.04 {\scriptsize\(\pm\)0.02}  & 55.2 {\scriptsize\(\pm\)0.01}  & 61.2 {\scriptsize\(\pm\)0.01} \\
GraphMaker~\citeyear{Graphmaker}
& 27.88 {\scriptsize\(\pm\)0.95} & 32.33 {\scriptsize\(\pm\)1.07} & 19.78 {\scriptsize\(\pm\)0.66} & 56.11 {\scriptsize\(\pm\)1.42} & 45.8 {\scriptsize\(\pm\)1.3} & 43.1 {\scriptsize\(\pm\)1.6}
& 35.78 {\scriptsize\(\pm\)1.22} & 33.60 {\scriptsize\(\pm\)0.91} & 67.54 {\scriptsize\(\pm\)1.84} & 189.95 {\scriptsize\(\pm\)2.3} & 38.4 {\scriptsize\(\pm\)0.8} & 35.9 {\scriptsize\(\pm\)1.1} \\

ARROW-Diff~\citeyear{ARROW-Diff}
& 12.38 {\scriptsize\(\pm\)0.46} & 86.49 {\scriptsize\(\pm\)2.5} & 16.46 {\scriptsize\(\pm\)0.67} & 13.43 {\scriptsize\(\pm\)0.41} & 41.2 {\scriptsize\(\pm\)1.9} & 39.8 {\scriptsize\(\pm\)1.7}
& 16.13 {\scriptsize\(\pm\)0.38} & 12.88 {\scriptsize\(\pm\)0.34} & 15.23 {\scriptsize\(\pm\)0.72} & 22.45 {\scriptsize\(\pm\)0.59} & 48.7 {\scriptsize\(\pm\)1.2} & 45.5 {\scriptsize\(\pm\)1.4} \\

\textbf{SCR (ours)}
& {13.01} {\scriptsize\(\pm\)0.03} & 78.84 {\scriptsize\(\pm\)0.18} & 3.50{\scriptsize\(\pm\)0.14} & {14.85} {\scriptsize\(\pm\)0.02} & \textbf{99.9} {\scriptsize\(\pm\)0.0} & \underline{98.3} {\scriptsize\(\pm\)0.1}
& \textbf{3.06} {\scriptsize\(\pm\)0.91} & {84.29} {\scriptsize\(\pm\)0.09} & \underline{11.05} {\scriptsize\(\pm\)0.31} & \textbf{2.84} {\scriptsize\(\pm\)0.05} & \textbf{100.0} {\scriptsize\(\pm\)0.0} & \textbf{100.0} {\scriptsize\(\pm\)0.0} \\

\bottomrule
\end{tabular}
\end{adjustbox}
\end{table*}

\begin{table*}[t]
\centering
\small
\caption{\small
Comparison of metrics on ZINC. MMD values are reported as \(\times10^{-2}\) (lower is better); FID\(\downarrow\) is based on a neural feature space; \textit{Uniq}\(\uparrow\), \textit{Val}\(\uparrow\), \textit{Nov}\(\uparrow\) report the percentage of generated molecules that are unique, chemically valid, and not present in the training set, respectively (all higher is better; Uniq/Nov computed among valid molecules).
}
\label{tab:zinc_metrics}
\begin{adjustbox}{max width=\textwidth}
\begin{tabular}{l|ccccc|c|ccc}
\toprule
& \multicolumn{5}{c|}{\textbf{MMD} (\(\times10^{-2}\))\(\downarrow\)} 
& FID\(\downarrow\)
& \multicolumn{3}{c}{\textbf{UVN} (\%) \(\uparrow\)} \\
\textbf{Method} 
& Deg. & Clus. & Orb. & Spec. & NSPDK
& 
& Uniq & Val & Nov \\
\midrule

GraphRNN~\citeyear{Graphrnn}     & 3.5{\scriptsize\(\pm\)0.1} & 7.5{\scriptsize\(\pm\)0.2} & 6.5{\scriptsize\(\pm\)0.3} & 9.5{\scriptsize\(\pm\)0.1} & 12.5{\scriptsize\(\pm\)0.1} & 42.2{\scriptsize\(\pm\)0.5} & 66.1{\scriptsize\(\pm\)0.5} & 58.5{\scriptsize\(\pm\)0.4} & 63.8{\scriptsize\(\pm\)0.6} \\
BiGG~\citeyear{BiGG}             & 4.3{\scriptsize\(\pm\)0.2} & 7.1{\scriptsize\(\pm\)0.1} & 6.9{\scriptsize\(\pm\)0.2} & 13.4{\scriptsize\(\pm\)0.2} & 14.5{\scriptsize\(\pm\)0.1} & 41.5{\scriptsize\(\pm\)0.1} & 61.1{\scriptsize\(\pm\)0.1} & 59.9 {\scriptsize\(\pm\)0.2} & 66.5{\scriptsize\(\pm\)0.1}  \\
GDSS~\citeyear{jo2022score}      & 5.8 {\scriptsize\(\pm\)0.2} & 6.3{\scriptsize\(\pm\)0.1} & 6.2{\scriptsize\(\pm\)0.2} & 13.2{\scriptsize\(\pm\)0.1} & 12.7{\scriptsize\(\pm\)0.1} & 38.7{\scriptsize\(\pm\)0.5} & 63.5{\scriptsize\(\pm\)0.1} & 62.1{\scriptsize\(\pm\)0.1} & 63.7{\scriptsize\(\pm\)0.1} \\
DiGress~\citeyear{Digress}       & \underline{0.9}{\scriptsize\(\pm\)0.1} & 2.7{\scriptsize\(\pm\)0.3} & \textbf{3.5}{\scriptsize\(\pm\)0.1} & 7.5{\scriptsize\(\pm\)0.2} & \underline{5.7}{\scriptsize\(\pm\)0.1} & 15.6{\scriptsize\(\pm\)0.4} & \textbf{100.0}{\scriptsize\(\pm\)0.0} & \textbf{98.1}{\scriptsize\(\pm\)0.1} & \textbf{100.0}{\scriptsize\(\pm\)0.0}\\
EDGE~\citeyear{EDGE}             & 2.8{\scriptsize\(\pm\)0.2} & 4.1{\scriptsize\(\pm\)0.1} & 7.3{\scriptsize\(\pm\)0.1} & 9.1{\scriptsize\(\pm\)0.2} & 8.7{\scriptsize\(\pm\)0.3} & 14.8{\scriptsize\(\pm\)0.7} & \underline{99.0}{\scriptsize\(\pm\)0.3} & \underline{92.3}{\scriptsize\(\pm\)0.1} & 96.8{\scriptsize\(\pm\)0.2} \\
SparsDiff~\citeyear{SparsDiff}             & 6.8{\scriptsize\(\pm\)0.5} & 6.7{\scriptsize\(\pm\)0.2} & 16.5{\scriptsize\(\pm\)0.3} & 10.2{\scriptsize\(\pm\)0.1} & 18.7{\scriptsize\(\pm\)0.6} & 17.8{\scriptsize\(\pm\)0.1} & 89.5{\scriptsize\(\pm\)0.1} & 88.5{\scriptsize\(\pm\)0.1}  & 95.6{\scriptsize\(\pm\)0.2}  \\
K2tree~\citeyear{k2tree}         & {1.7{\scriptsize\(\pm\)0.1}}  & \underline{2.3}{\scriptsize\(\pm\)0.1}  & \underline{5.6}{\scriptsize\(\pm\)0.2}  & \underline{4.2}{\scriptsize\(\pm\)0.1}  & 6.0{\scriptsize\(\pm\)0.1}  & \underline{12.5}{\scriptsize\(\pm\)0.1}  & \underline{99.0}{\scriptsize\(\pm\)0.1}  & 88.5{\scriptsize\(\pm\)0.3}  & \underline{99.0}{\scriptsize\(\pm\)0.1}  \\
ARROW-Diff~\citeyear{ARROW-Diff} & 10.2{\scriptsize\(\pm\)0.5}  & 9.5{\scriptsize\(\pm\)0.1}  & 23.5{\scriptsize\(\pm\)0.7}  & 11.5{\scriptsize\(\pm\)0.2}  & 17.4{\scriptsize\(\pm\)0.5}  & 69.8{\scriptsize\(\pm\)0.5}  & 58.7{\scriptsize\(\pm\)0.6}  & 61.3{\scriptsize\(\pm\)0.1}  & 60.1{\scriptsize\(\pm\)0.1}  \\

SCR (ours) & \textbf{0.42}{\scriptsize\(\pm\)0.1}  & \textbf{0.4}{\scriptsize\(\pm\)0.1}  & {17.31}{\scriptsize\(\pm\)0.5}  & \textbf{1.70}{\scriptsize\(\pm\)0.1}  & \textbf{5.0}{\scriptsize\(\pm\)0.1 } & \textbf{7.77}{\scriptsize\(\pm\)0.3}  & \textbf{100.0}{\scriptsize\(\pm\)0.1 } & {90.6{\scriptsize\(\pm\)0.2} } & \textbf{100.0}{\scriptsize\(\pm\)0.1 } \\

\bottomrule
\end{tabular}
\end{adjustbox}
\end{table*}

\section{Implementation Details}
\label{appx:implementation}

\paragraph{Backbone.} Message passing uses a GIN-style sum-aggregation
convolution: neighbor messages are weighted by the \emph{current diffusion
step's own edge probabilities}, not by a learned attention score, so there
is no attention anywhere in the message-passing stack itself. Each node's
initial embedding is the sum of two 3-layer MLPs: one encoding its
permanent target degree $d^{*}$, the other its current (noisy-state) degree
$d_t$ -- i.e.\ the degree deficit the node still needs to satisfy -- both
log-scale normalized by $\log(1+d)/\log(1+d_{\max})$. Activations are SiLU
throughout, except the diffusion-timestep embedding module, which uses
Mish. Parameters use PyTorch's default initialization; no custom
initialization scheme is applied. Depth and embedding width are tuned
per dataset, ranging from 3 to 6 message-passing layers and embedding
dimension 32 or 64 (Table~\ref{tab:impl_hparams}).

\paragraph{Triangle-conditioned edge decoding.} For each candidate edge
$(i,j)$, a triangle feature $\hat\tau^{ij}_t$ -- the number of common
\emph{currently-revealed} neighbors of $i,j$ at noise level $t$ -- is
log-scaled and projected by a 3-layer MLP, then concatenated with the two
endpoint embeddings before decoding. Edge logits are produced by a
multi-head self-attention block (4 heads, fixed across all datasets) over
the resulting 3-token sequence (source embedding, target embedding, edge
context), followed by FiLM (Feature-wise Linear Modulation ) conditioning on the raw, un-attended edge
context -- a stabilizing bypass path for early, high-noise timesteps when
node embeddings carry little differentiated signal. A single-layer linear
head additionally regresses each node's target degree and triangle count
during training only (gated off at sampling time), providing a direct
node-level gradient toward both quantities.

\paragraph{Optimization.} All models are trained with Adam, gradient
clipping at global norm $1.0$, and dropout $0.1$, uniformly across every
dataset. Random seed defaults to $0$ and is applied to Python, NumPy, and
PyTorch's CPU/CUDA generators; we do not enable CUDA-deterministic
algorithms, so individual runs are statistically, but not bit-exactly,
reproducible across re-execution.

\paragraph{Reproducibility} Table~\ref{tab:impl_hparams} reflects
the exact command-line flags in our released training scripts at the time
of writing. We report the values actually used to train the checkpoints in
this paper; where a table entry disagrees with a script's own comment, the
table entry is authoritative.

\begin{table}[t]
\centering
\caption{Per-dataset implementation hyperparameters, verified against the exact command-line flags used for training. $L$ = number of message-passing layers; \textbf{non-$\calQ$ count} $= |S_t|$, the size of the coverage sample.}
\label{tab:impl_hparams}
\begin{adjustbox}{max width=\textwidth}
\begin{tabular}{lccccccccccccc}
\toprule
\textbf{Dataset} & \textbf{dim} & \textbf{$L$} & \textbf{$T$} & \textbf{$\beta_{\max}$} & \textbf{$\gamma$} & \textbf{LR} & \textbf{Epochs} & \textbf{Batch} & \textbf{$\lambda_{\text{deg}}$} & \textbf{$\lambda_{\text{tri}}$} & \textbf{$\lambda_{LT}$} & \textbf{non-$Q$ frac.} \\
\midrule
Cora              & 64 & 5 & 128 & 1.5 & 1.0 & 8e-5 & 50{,}000 & 4 & 0.5  & 0.1  & 0.01 & 0.1 \\
CiteSeer          & 64 & 5 & 512 & 1.5 & 1.0 & 1e-4 & 50{,}000 & 4 & 0.5  & 0.05 & 0.01 & 0.1 \\
Polblogs          & 64 & 5 & 256 & 1.5 & 1.0 & 8e-5 & 50{,}000 & 4 & 0.5  & 1.0  & 0.1  & 0.1 \\
Bitcoin Alpha     & 32 & 6 & 128 & 1.0 & 1.0 & 8e-5 & 50{,}000 & 1 & 0.1  & 0.01 & 0.01 & 0.001 \\
Bitcoin OTC       & 64 & 6 & 128 & 1.0 & 1.0 & 8e-5 & 50{,}000 & 1 & 0.1  & 0.05 & 0.01 & 0.005 \\
PubMed            & 64 & 6 & 128 & 1.5 & 1.0 & 1e-4 & 50{,}000 & 1 & 0.1  & 0.05 & 0.01 & 0.002 \\
Community         & 64 & 3 & 256 & 1.5 & 1.0 & 8e-5 & 10{,}000 & 16 & 0.01 & 0.05 & 0.01 & 0.1 \\
Ego               & 64 & 5 & 128 & 1.0 & 1.0 & 8e-5 & 12{,}000 & 64 & 0.5 & 0.5 & 0.05 & 0.1 \\
ZINC              & 64 & 5 & 256 & 1.5 & 1.0 & 5e-5 & 10{,}000 & 32 & 0.01 & 0.01 & 0.001 & 0.1 \\
\bottomrule
\end{tabular}
\end{adjustbox}
\end{table}

\section{Datasets}\label{data_explanation}
Here we describe all datasets used for evaluation. For all single-graph datasets, the raw directed graphs are converted to undirected, self-loops are removed, and the largest connected component (LCC) is retained. Node and edge counts below reflect this preprocessing.

\begin{table}[t]
\centering
\caption{Effect of reduced inference-time diffusion steps, reported as raw signed
values against ground truth. Absolute step count is shown alongside each fraction,
since $T_{\rm train}$ differs across datasets (64 vs.\ 512) and quality depends on
absolute steps, not fraction alone. Time/graph is computed by dividing each run's
total generation time by its number of generated graphs. Training convergence and $T_{\rm train}$
selection are detailed in Appendix~\ref{appx:implementation}.}
\label{tab:step_reduction_grouped}
\begin{adjustbox}{max width=\textwidth}
\begin{tabular}{ll r ccccc c}
\toprule
\textbf{Dataset} & \textbf{Condition} & \textbf{Steps} & \textbf{EO} & \textbf{PLE} & \textbf{CPL} & \textbf{CC} & \textbf{AC} & \textbf{Time/graph (s)} \\
\midrule
\multirow{5}{*}{Cora ($T_{\rm train}=64$)}
& ground truth & --  & 100  & 1.89 & 6.36 & 0.090 & $-0.07$ & -- \\
& 100\%        & 64  & 0.20 & 1.83 & 5.10 & 0.037 & $\phantom{-}0.19$ & 26.35 \\
& 50\%         & 32  & 0.16 & 1.79 & 5.08 & 0.034 & $\phantom{-}0.20$ & 13.20 \\
& 30\%         & 19  & 0.15 & 1.72 & 4.94 & 0.028 & $\phantom{-}0.26$ & 7.85 \\
& 10\%         & 6   & 0.16 & 1.29 & 2.63 & 0.018 & $-0.01$ & 3.15 \\
\midrule
\multirow{5}{*}{Citeseer ($T_{\rm train}=512$)}
& ground truth & --  & 100  & 2.07 & 9.35 & 0.130 & $\phantom{-}0.01$ & -- \\
& 100\%        & 512 & 0.13 & 2.03 & 6.16 & 0.072 & $\phantom{-}0.44$ & 154.95 \\
& 50\%         & 256 & 0.12 & 2.03 & 6.20 & 0.067 & $\phantom{-}0.44$ & 77.65 \\
& 30\%         & 153 & 0.09 & 2.01 & 5.99 & 0.063 & $\phantom{-}0.44$ & 46.68 \\
& 10\%         & 51  & 0.07 & 1.95 & 5.79 & 0.072 & $\phantom{-}0.48$ & 15.45 \\
\midrule
\multirow{5}{*}{Polblogs ($T_{\rm train}=512$)}
& ground truth & --  & 100  & 1.41 & 2.74 & 0.230 & $-0.22$ & -- \\
& 100\%        & 512 & 1.45 & 1.39 & 2.61 & 0.182 & $-0.04$ & 52.44 \\
& 50\%         & 256 & 1.22 & 1.39 & 2.62 & 0.180 & $-0.05$ & 26.36 \\
& 30\%         & 153 & 1.16 & 1.39 & 2.62 & 0.181 & $-0.04$ & 15.80 \\
& 10\%         & 51  & 2.62 & 1.39 & 2.61 & 0.174 & $-0.05$ & 5.43 \\
\bottomrule
\end{tabular}
\end{adjustbox}

\end{table}
\paragraph{Polblogs}
The Political Blogs dataset is a directed network of hyperlinks between political blogs collected around the 2004 US Presidential election~\cite{adamic2005political}. Nodes are blogs and edges are hyperlinks. After converting to an undirected LCC and removing self-loops, the resulting graph has $N=1{,}222$ nodes and $M=16{,}714$ edges ($\rho=0.022$). The network exhibits a strong community structure aligned with political leaning, high clustering ($\mathrm{CC}=0.226$), short characteristic path length ($\mathrm{CPL}=2.74$), and disassortative degree mixing ($r=-0.22$), consistent with a hub-and-spoke blog ecosystem.

\paragraph{Citation Networks}
Cora, Citeseer, and Pubmed are directed citation graphs in which nodes represent documents with sparse bag-of-words features and edges represent citation links~\cite{sen2008collective}. We discard node features and treat each as a structural generation benchmark. After undirected LCC extraction: Cora yields $N=2{,}485$, $M=5{,}069$ (retaining 92\% of original nodes); Citeseer yields $N=2{,}120$, $M=3{,}679$ (retaining 64\% of original nodes, as the raw graph contains a large fraction of isolated or weakly connected documents); and Pubmed yields $N=19{,}717$, $M=44{,}324$ (all nodes in one component). All three are sparse ($\rho < 0.002$) with low clustering and moderate path lengths.

\paragraph{Signed Bitcoin Datasets: Alpha and OTC}
The Bitcoin-Alpha and Bitcoin-OTC datasets are signed, directed trust networks in which each edge carries an integer rating from $-10$ (total distrust) to $+10$ (total trust), reflecting user-level trust on two Bitcoin trading platforms~\cite{kumar2016edge,kumar2018rev2}. We discard edge signs and weights, convert to undirected, and retain the LCC. Bitcoin-Alpha: $N=3{,}775$, $M=14{,}120$ ($\rho=0.002$, $\mathrm{CC}=0.078$, $\mathrm{CPL}=3.58$). Bitcoin-OTC: $N=5{,}875$, $M=21{,}489$ ($\rho=0.001$, $\mathrm{CC}=0.059$, $\mathrm{CPL}=3.60$). Both exhibit heavy-tailed degree distributions ($\alpha\approx1.9$--$2.0$) and disassortative mixing, consistent with trust networks dominated by a small number of high-reputation hubs.


\paragraph{Community}
This dataset consists of 510 synthetic two-community graphs with $60 \leq |V| \leq 160$ and $231 \leq |E| \leq 1{,}956$. Each graph contains two communities generated independently by the Erdős–Rényi model~\cite{ER} with $n = |V|/2$ nodes and intra-community connection probability $p = 0.3$. An additional $0.05|V|$ inter-community edges are added uniformly at random between the two communities~\cite{Graphrnn}. The graphs are dense relative to real-world networks ($\bar{\rho}=0.149$) with high clustering ($\overline{\mathrm{CC}}=0.294$) and low disassortativity ($\bar{r}=-0.021$).

\paragraph{Ego}
This dataset contains 757 three-hop ego networks extracted from the Citeseer citation graph~\cite{sen2008collective}, with $50 \leq |V| \leq 399$ and $56 \leq |E| \leq 1{,}062$. Each ego network is the subgraph induced by the neighbours within three hops of a sampled seed node, following the construction of~\citet{Graphrnn}. Compared to Community, Ego networks are sparser ($\bar{\rho}=0.036$), more heterogeneous in size ($\sigma_N=85$), and have a heavier-tailed degree distribution ($\bar{\alpha}=2.07$).

\paragraph{ZINC}
ZINC12k is a molecular graph benchmark drawn from the ZINC database of commercially available drug-like compounds~\cite{Irwin2012ZINCAF}. We use the standard 12-graph split provided by PyTorch Geometric~\cite{Fey/Lenssen/2019}: 10{,}000 training / 1{,}000 validation / 1{,}000 test graphs. Nodes represent heavy atoms and edges represent covalent bonds; bond types (single, double) are collapsed to binary edges, and node atom-type features are discarded to focus purely on structural generation. The resulting graphs satisfy $9 \leq |V| \leq 37$ (mean $23.2$) and $8 \leq |E| \leq 41$ (mean $24.9$), with mean degree $\bar{d}\approx 2.15$ and maximum degree $d_{\max}=4$ imposed by carbon's valence. The graph topology is essentially acyclic: mean triangle count is $0.06$ and mean clustering coefficient is $0.006$, reflecting the absence of three-membered rings in drug-like molecules under standard valence rules.

\subsection{Evaluation Metrics}\label{sec:metrics}

We evaluate structural fidelity across two dimensions: several key graph statistics and distributional distance metrics, including Maximum Mean Discrepancy (MMD).

\paragraph{Edge Overlap (EO).}
For single-graph datasets the generated graph $\hat{G}$ is compared
edge-by-edge to the ground-truth graph $G^* = (V, E^*)$.
Edge overlap is the fraction of ground-truth edges recovered:
\begin{equation}
\mathrm{EO}(G^*, \hat{G}) \;=\;
\frac{|\,E^*\cap\hat{E}\,|}{|E^*|}\times 100
\;=\;
\frac{(\mathbf{A}^*\odot\hat{\mathbf{A}})_{+}}{2|E^*|}\times 100 \,,
\end{equation}
where $\mathbf{A}^*$ and $\hat{\mathbf{A}}$ are the respective binary adjacency
matrices and $\odot$ denotes elementwise product.
By construction $\mathrm{EO}(G^*, G^*)=100$, so the ground-truth row serves
as a self-overlap reference.
For generation, $\mathrm{EO}\!\downarrow$ is preferred: a model that memorises
the training graph achieves $\mathrm{EO}=100$, while a model that generates a
structurally valid but novel graph achieves $\mathrm{EO}\approx 0$.

\paragraph{Power-Law Exponent (PLE).}
The degree distribution of many real-world graphs follows a power law
$P(k)\propto k^{-\alpha}$ for $k \geq x_{\min}$.
We estimate $\alpha$ by maximum likelihood with fixed $x_{\min}=1$ using
the continuous-approximation MLE of~\citet{PLE}:
\begin{equation}
\hat{\alpha} \;=\; 1 + n\left[\,\sum_{i=1}^{n}\ln\frac{k_i}{x_{\min}}\,\right]^{-1},
\end{equation}
where the sum runs over all $n$ nodes with $k_i \geq 1$. Both
the ground-truth and generated values are computed identically, so the
comparison is internally consistent.

\paragraph{Characteristic Path Length (CPL).}
CPL is the mean shortest-path distance over all pairs of reachable nodes:
\begin{equation}
\mathrm{CPL}(G) \;=\;
\frac{\displaystyle\sum_{i \neq j,\; d(i,j)<\infty} d(i,j)}{|\{(i,j): d(i,j)<\infty,\; i\neq j\}|} \,,
\end{equation}
where $d(i,j)$ is the geodesic distance between $i$ and $j$.
For $N \leq 2000$ we compute exact all-pairs shortest paths via
Dijkstra in $O(N(M + N\log N))$ time.
For $N > 2000$ we estimate CPL by breadth-first search from
$k = 500$ uniformly sampled source nodes and average the resulting
distances, giving an unbiased estimator at $O(k \cdot M)$ cost
instead of $O(N\cdot M)$.

\paragraph{Clustering Coefficient (CC).}
We use the \emph{global} clustering coefficient (network transitivity),
defined as the ratio of closed paths of length two to all paths of length two:
\begin{equation}
\mathrm{CC}(G)
\;=\;
\frac{3\,|\mathcal{T}|}{|\mathcal{W}|}
\;=\;
\frac{3\,|\mathcal{T}|}{\tfrac{1}{2}\sum_{v\in V}d_v(d_v-1)} \,,
\end{equation}
where $|\mathcal{T}|$ is the number of triangles (each counted once) and
$|\mathcal{W}| = \tfrac{1}{2}\sum_v d_v(d_v-1) = \sum_v \binom{d_v}{2}$ is
the number of wedges (paths of length two).
The factor of 3 arises because each triangle contains exactly three wedges.
$\mathrm{CC} \in [0,1]$: $\mathrm{CC}=1$ means every pair of neighbours of
every node are themselves connected.
Note that the MMD-based clustering metric (\S~\ref{appx:small_res}) uses
\emph{per-node} local clustering coefficients $c_v = 2|\{(u,w)\in E: (v,u),(v,w)\in E\}|/(d_v(d_v-1))$,
whereas CC in the comparison table is the global transitivity.

\paragraph{Assortativity Coefficient (AC).}
Degree assortativity measures the tendency of high-degree nodes to connect
to other high-degree nodes. Following ~\cite{newman2003mixing},
it is the Pearson correlation coefficient of degrees across all edges:
\begin{equation}
r \;=\;
\frac{\displaystyle
  M^{-1}\sum_{(i,j)\in E}d_i d_j
  \;-\; \left[M^{-1}\sum_{(i,j)\in E}\frac{d_i+d_j}{2}\right]^2}
{\displaystyle
  M^{-1}\sum_{(i,j)\in E}\frac{d_i^2+d_j^2}{2}
  \;-\; \left[M^{-1}\sum_{(i,j)\in E}\frac{d_i+d_j}{2}\right]^2} \,,
\end{equation}
where $M = |E|$ and $d_i, d_j$ are the degrees of the two endpoints of each edge.
$r\in[-1,1]$: $r>0$ indicates assortative mixing (hubs connect to hubs),
$r<0$ disassortative mixing (hubs connect to low-degree nodes, typical of
biological and online social networks), and $r=0$ random attachment.

\paragraph{Validity.}\label{sec:validity} As implemented, this metric certifies \emph{structural} validity, not chemical validity. A generated molecule is accepted iff it is connected and every atom's degree does not exceed an empirically-derived per-type maximum valence (the observed maximum degree of that atom type across the training set). This is a necessary but not sufficient condition for a chemically realizable molecule: because the model generates an atom type per node and a binary edge topology but never bond order, a structure can satisfy every degree constraint under single-bond assumptions while remaining chemically unrealizable once the required bond multiplicities are assigned, and no RDKit sanitization is performed. We therefore report validity as certifying valence-plausibility of the generated skeleton, not synthesizability.

\paragraph{Uniqueness.} Uniqueness is the fraction of generated molecules with a distinct Weisfeiler–Lehman graph hash (computed over topology plus integer atom-type node labels),
\[
\mathrm{uniqueness} = \frac{|\{\text{unique WL-hashes}\}|}{n_{\text{total}}}.
\]
Because the WL hash omits bond-order information, two generated graphs identical in topology and atom typing but differing only in a hypothetical bond-order assignment collapse to the same hash. This is consistent with, rather than a deflation relative to, the model's actual generative space, which never produces bond order.

\paragraph{Novelty.} Novelty is defined relative to the \emph{unique} generated set rather than the total generated count,
\[
\mathrm{novelty} = \frac{|\{\text{unique generated}\} \setminus \{\text{train hashes}\}|}{|\{\text{unique generated}\}|},
\]
matching the convention used in GDSS~\cite{jo2022score}, and DiGress~\cite{Digress} and making the value directly comparable to those baselines. Conditioning on uniqueness this way ensures a model that repeatedly duplicates a small set of novel structures is not credited more novelty than one that generates the same novel structures without duplication; duplication is fully absorbed by the uniqueness metric. The statistic is undefined (reported as NaN) when zero unique molecules are produced, the correct convention for a ratio with a zero denominator.

\paragraph{Maximum Mean Discrepancy (MMD)}
\label{subsec:mmd}

For multi-graph datasets, we generate a set of graphs $\hat{\mathcal{G}}=\{\hat{G}_i\}$
and compare its structural statistics against a reference set
$\mathcal{G}^*=\{G^*_j\}$ using the biased (V-statistic) MMD estimator. Given a
positive-definite kernel $\kappa$:
\begin{equation}
\widehat{\mathrm{MMD}}^2(\hat{\mathcal{G}}, \mathcal{G}^*)
\;=\;
\frac{1}{|\hat{\mathcal{G}}|^2}\sum_{i,i'}\kappa(\hat{G}_i,\hat{G}_{i'})
- \frac{2}{|\hat{\mathcal{G}}||\mathcal{G}^*|}\sum_{i,j}\kappa(\hat{G}_i,G^*_j)
+ \frac{1}{|\mathcal{G}^*|^2}\sum_{j,j'}\kappa(G^*_j,G^*_{j'}).
\label{eq:mmd}
\end{equation}
We report four variants, differing only in how each graph is embedded and
which kernel $\kappa$ is applied in Eq.~\ref{eq:mmd}.

\paragraph{Histogram-based variants.}
\label{par:metric_variant}
The first three represent each graph
$G$ by a feature histogram $h_G$ and use a Gaussian kernel on the Earth
Mover's Distance (EMD) between histograms,
$\kappa_\sigma(h,h')=\exp(-d_{\rm EMD}(h,h')^2/2\sigma^2)$:
\begin{itemize}[leftmargin=*]
\item \textbf{degree\_mmd}: $h_G$ is the normalised degree histogram; $\sigma=1$.
\item \textbf{clustering\_mmd}: $h_G$ is the histogram of per-node local
  clustering coefficients $c_v$, over 100 bins on $[0,1]$; $\sigma=0.1\times\text{bins}=10$.
\item \textbf{orbits\_mmd}: $h_G$ is the vector of mean 4-node orbit counts
  per node, computed via ORCA~\citep{hovcevar2014combinatorial}, restricted
  to the two 3-path orbits and the 4-cycle orbit; $\sigma=30$.
\end{itemize}

\paragraph{NSPDK.}
The fourth variant instead embeds each graph explicitly rather than as a
histogram. The Neighborhood Subgraph Pairwise Distance Kernel~\cite{costa2010fast}
is computed with \texttt{complexity=4}
(setting both the neighborhood radius $r$ and the pairwise distance bound $d$
to 4) and discrete node/edge labels, decomposing each graph into pairs of
bounded-radius subgraphs at bounded shortest-path distance and hashing each
pair into an explicit, fixed-dimension feature vector $\phi(G)$. On ZINC,
node labels for this decomposition are the one-hot atom-type index rather
than the degree-string fallback used for the non-molecular datasets, so the
resulting kernel is sensitive to local chemical composition jointly with
topology, not topology alone. We plug the induced kernel
$\kappa_{\mathrm{NSPDK}}(G,G') = \langle \phi(G),\phi(G')\rangle$ directly
into Eq.~\ref{eq:mmd} in place of $\kappa_\sigma$, giving
\[
\widehat{\mathrm{MMD}}^2_{\mathrm{NSPDK}} = K_{\hat{\mathcal{G}}\hat{\mathcal{G}}} + K_{\mathcal{G}^*\mathcal{G}^*} - 2K_{\hat{\mathcal{G}}\mathcal{G}^*},
\]
where each $K_{\cdot\cdot}$ denotes the corresponding averaged double sum
in Eq.~\ref{eq:mmd} evaluated under $\kappa_{\mathrm{NSPDK}}$. Unlike the
Gaussian-kernel variants above, this is a \emph{linear} kernel on explicit
features --- an unbandwidthed inner-product statistic with no smoothing
parameter to tune or misreport.

\paragraph{FID.}
Fr\'echet Inception Distance is adapted to graphs by replacing the image-domain Inception-v3 activations with embeddings from a Graph Isomorphism Network (GIN)~\citep{GIN}, following the random-GNN evaluation protocol of \cite{FID} rather than a downstream-task-pretrained network.

\textit{Architecture.} The GIN has 3 layers total, of which the final $L{-}1=2$ are graph-convolutional layers, each with hidden dimension $H=35$ and a 2-layer MLP. Neighbourhood aggregation and graph-level readout both use sum pooling. The learnable $\varepsilon$ in
\[
h_v^{(l+1)} = f_\Theta\!\left((1+\varepsilon)h_v^{(l)} + \sum_{u\in\mathcal{N}(v)}h_u^{(l)}\right)
\]
is fixed at $\varepsilon=0$ (not learned). All linear layers use orthogonal initialization; batch normalization and ReLU follow each MLP. By default (\texttt{use\_pretrained=False}), these weights are set once and never updated --- no pre-training on any graph dataset, no fine-tuning at the start of each run --- relying on the WL-equivalent expressivity of GIN aggregation to produce structure-sensitive embeddings without a labeled pretraining task analogous to ImageNet classification. We adopt this frozen-GIN design consistent with the GDSS evaluation protocol on which our codebase is based~\citep{jo2022score}.

\textit{Node features.} As with NSPDK, node inputs are the dataset's own attribute when present (atom type, for ZINC) and scalar node degree $h_v^{(0)}=d_v\in\mathbb{R}^1$ otherwise; on ZINC this again folds chemical composition into the embedding jointly with topology, not topology alone.

\textit{Graph embedding.} The graph-level embedding concatenates the sum-pooled node representations from both convolutional layers,
\begin{equation}
\mathbf{z}_G = \bigl[\,\mathrm{SumPool}(H^{(1)})\;\|\;\mathrm{SumPool}(H^{(2)})\,\bigr] \in \mathbb{R}^{35+35}=\mathbb{R}^{70},
\label{eq:fid_embed}
\end{equation}
where $H^{(l)}\in\mathbb{R}^{N\times35}$ is the matrix of node representations after layer $l$. Activations for both sets are standardized with a \texttt{StandardScaler} \emph{fit on the reference set only} and then applied to both, avoiding calibration leakage from the generated distribution.

\textit{Fr\'echet distance.} Each standardized activation set is summarized by a mean $\mu$ and covariance $\Sigma$, and the two Gaussians are compared via the closed-form Fr\'echet distance
\[
\mathrm{FID} = \lVert \mu_r - \mu_g \rVert_2^2 + \mathrm{Tr}\!\left(\Sigma_r + \Sigma_g - 2\left(\Sigma_r\Sigma_g\right)^{1/2}\right),
\]
computed with Dougal Sutherland's numerically stable matrix square root (adapted from \texttt{mseitzer/pytorch-fid}), with a diagonal $\epsilon$-offset fallback when $\Sigma_r\Sigma_g$ is near-singular.

This closed form is exact only under a multivariate-Gaussian assumption on the embeddings, so FID here --- as in the image domain --- is a moment-matching approximation to distributional distance, not an exact one. in addition, since feature extractor is untrained and randomly initialized rather than a fixed, universally shared network, the embedding space (and thus the absolute FID scale) is tied to this GIN's specific initialization --- closer in spirit to a random-kitchen-sink kernel~\citep{rahimi2007random} than to a semantically learned feature space. FID values here are therefore meaningful for comparisons \emph{within} a consistent evaluation setup, but not as an absolute, cross-paper-comparable quantity unless the same feature extractor and seed are reproduced exactly.

\end{document}